\documentclass[letterpaper]{article}
\usepackage{aaai2027-arxiv}
\usepackage[hyphens]{url}
\usepackage{graphicx}
\usepackage[numbers,sort&compress]{natbib}
\usepackage{caption}
\DeclareCaptionStyle{ruled}{labelfont=normalfont,labelsep=colon,strut=off}

\usepackage{booktabs}
\usepackage{xcolor}
\usepackage{placeins}
\usepackage{dblfloatfix}
\usepackage{algorithm}
\usepackage{algorithmic}

\usepackage{amsmath}
\usepackage{amssymb}
\usepackage{amsthm}
\newtheorem{theorem}{Theorem}
\newtheorem{lemma}{Lemma}
\newtheorem{proposition}{Proposition}
\newtheorem{corollary}{Corollary}
\newtheorem{assumption}{Assumption}
\theoremstyle{remark}
\newtheorem{remark}{Remark}
\usepackage[hidelinks]{hyperref}

\nocopyright

\ifPDFTeX
\else
  \hypersetup{
    pdftitle={InfluenceField: A Differentiable Field with Interventionally Identifiable Causal Structure for Multimodal World Modeling},
    pdfauthor={Zihao Yang, Zijia Wang, Zhiqiu Huang}
  }
\fi


\newcommand{\R}{\mathbb{R}}
\newcommand{\E}{\mathbb{E}}
\newcommand{\GL}{\mathrm{GL}}
\newcommand{\supp}{\mathrm{supp}}
\newcommand{\diag}{\mathrm{diag}}
\newcommand{\norm}[1]{\left\|#1\right\|}
\newcommand{\Ah}{\hat A}
\newcommand{\Ws}{W_\sigma}

\newcommand{\resultpm}[2]{$#1\pm#2$}
\newcommand{\resultpmc}[2]{$#1{\scriptstyle\pm#2}$}

\title{InfluenceField: A Differentiable Field with Interventionally Identifiable Causal Structure for Multimodal World Modeling}

\author{
    Zihao Yang\equalcontrib
    \thanks{Corresponding author: shil6682@ox.ac.uk}
    \textsuperscript{\rm 1},
    Zijia WANG\equalcontrib\textsuperscript{\rm 2},
    Zhiqiu HUANG\textsuperscript{\rm 3}
}

\affiliations{
    \textsuperscript{\rm 1}Mathematical Institute, University of Oxford\\
    \textsuperscript{\rm 2}Department of Computer Science, University of Oxford\\
    \textsuperscript{\rm 3}School of Mathematical Sciences,
    University of Nottingham Ningbo China
}

\begin{document}

\maketitle

\begin{abstract}
Multimodal large language models often capture visual--linguistic correlations
but struggle to predict how local visual interventions propagate and affect
downstream answers. We introduce InfluenceField, an intervention-aware latent
field inserted between the visual encoder and language decoder. It lifts patch
features into a continuous spatial representation, propagates directed
influence over multiple steps, and predicts local intervention effects through
a shared transition operator. Training jointly optimizes language modeling,
cross-environment invariance, counterfactual rollout supervision, and
structural regularization. For a nonlinear finite-basis population model, we
show that target-aligned interventional supervision, together with a one-step
separation condition on the transition, restricts admissible representations
to within-location reparameterizations, so that the directed dependency graph
of the full transition is recovered exactly. A linear specialization gives an
exact partial-coverage characterization and a finite-loss stability bound, and
the field analysis derives the spatial profile of coefficient interventions
together with a shared-channel calibration result.

On CausalVQA, InfluenceField improves overall accuracy over its backbone by
13.1 percentage points, with the largest gains on the planning and
hypothetical categories. Capacity-matched baselines and structural controls
attribute the gains in robustness and factual--counterfactual consistency to
the causal objectives rather than to added capacity.
\end{abstract}

\section{Introduction}
\label{sec:intro}
An intelligent multimodal system should not only recognize what happens, but also explain why it happens and predict what would happen if conditions changed \citep{pearl2009,chen2024cello}. Such capabilities help make multimodal systems more interpretable, robust, and controllable. Video understanding is particularly demanding because models must integrate spatial content, temporal dependencies, and vision–language interactions \citep{lin2024videollava,maaz2024videochatgpt}. However, current multimodal systems still rely heavily on superficial correlations and remain unreliable under context shifts \citep{chen2024cello,li2025mucr}.  Asked why
the ground is wet, a model may answer ``because the person has an
umbrella,'' mistaking correlation for cause (Fig.~\ref{fig:teaser}a).

\begin{figure}[!htbp]
\centering
\includegraphics[width=\columnwidth]{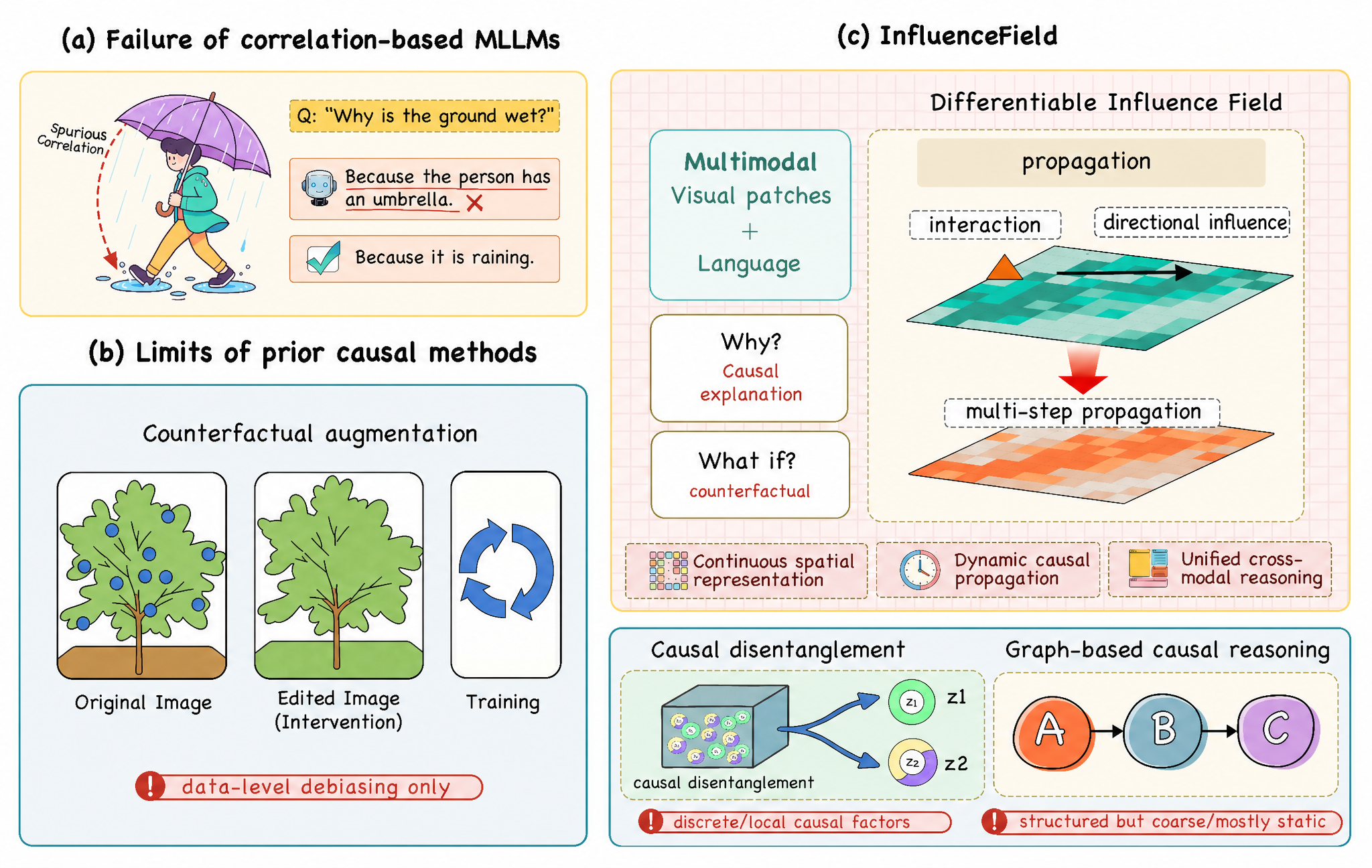}
\caption{(a) Multimodal models often rely on spurious correlations rather than causal mechanisms. (b) Existing causality-guided methods and their limitations. (c) InfluenceField provides unified cross-modal reasoning with a continuous spatial representation and dynamic influence propagation, trained with intervention-aware objectives.}
\label{fig:teaser}
\end{figure}

This limitation is particularly evident in counterfactual reasoning,
where a model must predict the consequences of an intervention in a
specific observed situation rather than extrapolate from observational
correlations. Consider a video in which a ball strikes a row of
dominoes and they fall in sequence. When asked what would have happened
if the middle domino had been removed, the model must trace how this
local intervention would break the chain of spatial interactions and
prevent the downstream dominoes from falling, rather than relying on a
learned correlation between a moving ball and the entire row falling.
Recent approaches incorporate causal information into multimodal
learning through counterfactual data augmentation
\citep{zhang2025cfvlm}, representation-level causal disentanglement
through confounder adjustment \citep{hu2025causalllava}, and multimodal
causal discovery \citep{li2025mllmcd}. However, these approaches do not jointly provide
an explicit spatial state in which local intervention effects can
propagate across regions over multiple reasoning steps.

Beyond these methodological differences, a key theoretical 
question is whether learned causal representations and structures are identifiable. This question is central to causal representation learning \citep{ahuja2023,lachapelle2022}, but remains largely unaddressed in existing multimodal causal methods \citep{hu2025causalllava,li2025mllmcd}.

We propose \textbf{InfluenceField}, a differentiable influence field inserted between the visual encoder and the language decoder. It unifies continuous spatial modeling, multi-step propagation of directed influence, and localized field-space interventions, and is trained with a joint objective combining language modeling, cross-environment invariance, counterfactual rollout supervision, and structural regularization. We also establish sufficient conditions under which target-aligned interventions identify the dependency graph of a nonlinear finite-basis transition up to within-location reparameterization, and a linear specialization that yields finite-coverage and finite-loss guarantees (Sec.~\ref{sec:theory}). Our contributions are:

\begin{itemize}
\item \textbf{Continuous field with directional propagation.} Patch features are lifted into a continuous latent field queryable at arbitrary locations; a state-dependent directional influence matrix with an iterative propagation operator lets local interventions induce multi-step, long-range effects.
\item \textbf{Identifiability analysis.} In a nonlinear finite-basis population
model, target-aligned intervention consistency preserves the complete
transition graph, while the representation is determined up to
within-location diffeomorphisms. The linear specialization gives an exact
partial-coverage characterization and a finite-loss stability bound, and the
field analysis derives the kernel-column intervention profile and its
shared-channel calibration condition.
\item \textbf{Capacity-matched and structure-destruction controls.} Parameter/FLOP-matched non-causal controls, influence-matrix transposition/symmetrization, and shuffled counterfactual supervision isolate the causal contribution; coverage curves, cross-seed stability, operational-influence agreement, bandwidth effects, and the observed $\sqrt{\varepsilon}$ trend probe theory-motivated predictions.
\item \textbf{Empirical validation.} Across causal video QA, multimodal causal discovery, OOD generalization, and counterfactual consistency, InfluenceField outperforms strong baselines, with the largest gains on planning and hypothetical questions.
\end{itemize}

\begin{figure*}[!htbp]
\centering
\includegraphics[width=0.96\textwidth]{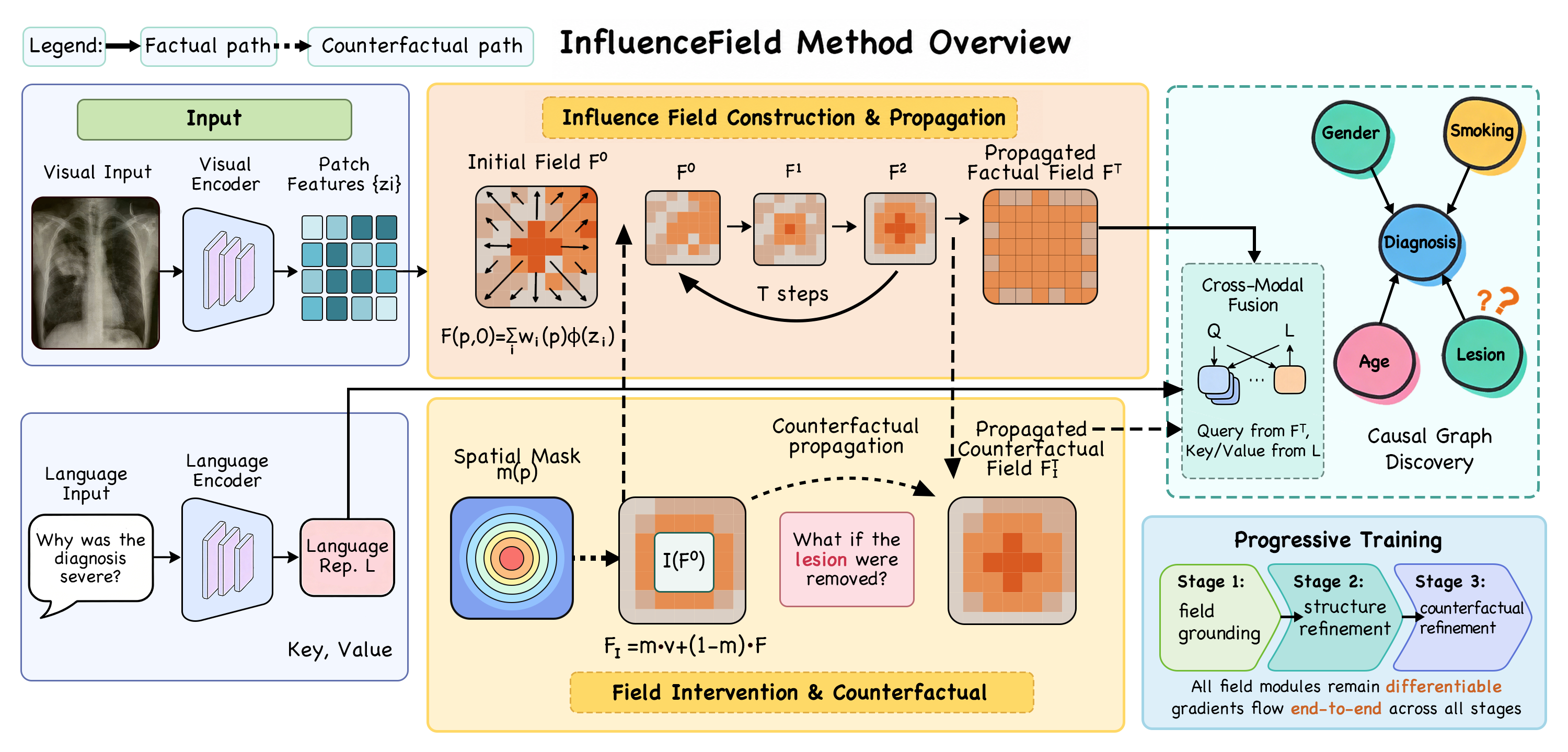}
\caption{Overview of InfluenceField. Visual patches are lifted into a continuous field $\mathcal{F}^0$ by Gaussian interpolation (Eq.~\ref{eq:field-init}), iteratively propagated for $T$ steps via the learned directional influence matrix (Eqs.~\ref{eq:influence}--\ref{eq:propagation}), and optionally intervened upon with a kernel-profile mask (Eq.~\ref{eq:field-int}) whose counterfactual rollout shares the same transition operator. Factual or counterfactual fields are fused with language tokens via cross-modal attention for answer generation; training combines language supervision with invariance, counterfactual rollout, and structural objectives (Eq.~\ref{eq:loss-total}).}
\label{fig:method}
\end{figure*}

\section{Related Work}
\label{sec:related}

\paragraph{Causality-guided multimodal reasoning.}
\emph{Data-level counterfactual augmentation:}
CF-VLM \citep{zhang2025cfvlm} encodes causal supervision through
paired counterfactual examples and specialized objectives, but leaves
the causal structure implicit in the learned representation rather
than instantiating an explicit causal state and transition model.
\emph{Representation-level causal disentanglement:}
Causal-LLaVA \citep{hu2025causalllava} performs causal disentanglement
through confounder adjustment using finite, category-indexed visual and
textual confounder dictionaries; it does not model how intervention
effects propagate across spatial regions over multiple rollout steps.
\emph{Structure-level causal discovery:}
COAT and MLLM-CD \citep{liu2024coat,li2025mllmcd} infer or refine
directed relations from unstructured or multimodal observations.
These methods represent causal structure as factor-level graphs rather
than continuous spatial states governed by a transition operator.
Taken together, these representative approaches neither integrate a
continuous spatial causal state with multi-step language-conditioned
reasoning nor establish conditions for recovering an internal spatial
influence structure.

For evaluation, CELLO \citep{chen2024cello}, CausalVLBench
\citep{komanduri2025}, MuCR \citep{li2025mucr}, and CausalSpatial
\citep{ma2026causalspatial} primarily assess causal reasoning in
image-based settings, whereas CausalVQA \citep{foss2025causalvqa}
focuses on physically grounded videos in which causal effects unfold
over time and serves as our main benchmark.

\paragraph{Causal representation learning and identifiability.}
Identifiability of latent causal representations generally requires
additional supervision or structural assumptions, including
interventions, temporal or auxiliary variables, and mechanism
constraints
\citep{scholkopf2021,ahuja2023,lachapelle2022,hyvarinen2019,khemakhem2020}.
Representative approaches include CausalVAE, CITRIS, and BISCUIT
\citep{yang2020causalvae,lippe2022citris,lippe2023biscuit}.
Prior work establishes nonlinear identification from paired interventions or
temporal supervision, recovering latents up to relabeling and within-variable
reparameterization and handling multidimensional causal factors
\citep{brehmer2022,lippe2022citris}. Complementary work obtains affine or
permutation-and-scaling equivalence classes by exploiting interventional
support geometry or linear latent causal models under linear and general
nonlinear observation maps \citep{ahuja2023,squires2023,buchholz2023}. Our
analysis tailors these ideas to the deterministic finite-basis transition used
for field rollout: the nonlinear result preserves the block support of the
Jacobian of the full transition map, while the linear specialization
characterizes partial coverage, finite-loss stability, and field-induced
intervention geometry.

\paragraph{World models and continuous scene representations.}
Canonical world models learn action-conditioned latent dynamics for prediction and control \citep{ha2018,hafner2019,hafner2020}. However, predictive fit under the observed regime need not transfer to forecasts under interventions \citep{vankadara2022}. Causal dynamics have been studied in object-oriented reinforcement-learning environments \citep{yu2024}, while recent work couples an LLM agent with a CRL-based causal world model for multi-step reasoning and planning \citep{gkountouras2025}. The latter operates on disentangled factor-level variables through a modular simulator. NeRF \citep{mildenhall2021nerf} and D-NeRF \citep{pumarola2021dnerf} provide differentiable continuous-coordinate representations but are not designed for causal interventions or directed influence propagation. InfluenceField shares the multi-step rollout paradigm, but additionally embeds a continuous spatial state with directed influence and localized interventions directly in the visual--language pathway for language-grounded counterfactual reasoning.

\section{Methodology}
\label{sec:method}

\subsection{Overview}
Given multimodal inputs, InfluenceField (Fig.~\ref{fig:method}) learns one shared field representation through four jointly trained components---\emph{field construction}, a \emph{directional influence function}, a \emph{propagation operator}, and a \emph{field intervention module}---supporting factor localization, directed influence inference, and intervention-aware counterfactual prediction.

\subsection{Field Construction}
Let the visual encoder output $N$ patch features $\{z_i\}_{i=1}^N$, $z_i\in\mathbb{R}^C$, at normalized positions $p_i\in[0,1]^2$. We define a field $\mathcal{F}:[0,1]^2\times\{0,\dots,T\}\to\mathbb{R}^C$ initialized by Gaussian interpolation
\begin{equation}
\mathcal{F}(p,0)=\sum_{i=1}^N w_i(p)\,\phi(z_i),
\label{eq:field-init}
\end{equation}
where $\phi$ is a learnable projection shared across patches and
\begin{equation}
w_i(p)=\frac{\exp(-\|p-p_i\|^2/2\sigma^2)}{\sum_{j=1}^N \exp(-\|p-p_j\|^2/2\sigma^2)}
\label{eq:field-kernel}
\end{equation}
is a normalized Gaussian kernel with bandwidth $\sigma>0$. The construction is continuous (queryable at arbitrary positions), differentiable (gradients flow into the field and encoder interface), and smooth (nearby positions share feature mixtures, aiding occlusion interpolation and consistency across object parts). In practice we evaluate the field on a regular $H\times W$ grid, $\mathbf{F}^0\in\mathbb{R}^{K\times C}$, $K=HW$; in our implementation the field grid coincides with the patch grid ($K=N$), so the sampled interpolation matrix is square. 

\subsection{Directional Influence and Propagation}
Given the discretized field $\mathbf{F}^t$, the directed influence from location $i$ to $j$ is
\begin{equation}
G^t_{ij}=S^t_{ij}\cdot\frac{1}{n_h}\sum_{h=1}^{n_h}A^{t,h}_{ij},
\label{eq:influence}
\end{equation}
where $A^{t,h}_{ij}$ is scaled dot-product attention affinity with a learnable directional bias,
\begin{equation}
e^{t,h}_{ij}=\frac{(\mathbf{W}^h_Q\mathbf{F}^t_i)^{\!\top}(\mathbf{W}^h_K\mathbf{F}^t_j)}{\sqrt{d_h}}+b^{\mathrm{dir}}_h,\;
A^{t,h}_{ij}=\frac{\exp(e^{t,h}_{ij})}{\sum_k \exp(e^{t,h}_{ik})},
\label{eq:attn}
\end{equation}
and the sparsity gate $S^t_{ij}=\sigma(\mathrm{MLP}[\mathbf{F}^t_i\,\|\,\mathbf{F}^t_j])$ suppresses weak or implausible edges. The matrix is row-normalized, $\hat G^t_{ij}=G^t_{ij}/(\sum_k G^t_{ik}+\epsilon)$; asymmetry arises from distinct query/key projections and the pairwise gate, so generally $G^t_{ij}\neq G^t_{ji}$, and $\mathbf{G}^t$ evolves with the scene state since it is recomputed each step. Effects propagate via
\begin{equation}
\mathbf{F}^{t+1}=\alpha\,(\hat{\mathbf{G}}^t)^{\!\top}\mathbf{F}^t+(1-\alpha)\,\widetilde{\mathbf{F}}^t,
\label{eq:propagation}
\end{equation}
where $\widetilde{\mathbf{F}}^t=\mathbf{W}_O(\mathrm{MHA}(\mathbf{F}^t))$ is an attention-based refinement branch and $\alpha=\sigma(w_{\mathrm{mix}})\in(0,1)$ is learnable: the first term performs explicit influence-based transport, the second preserves contextual semantics beyond the sparse graph. After $T$ steps, $\mathbf{F}^T=\mathcal{T}^T(\mathbf{F}^0)$; multi-step propagation lets indirect dependencies emerge, with depth playing the role of latent rollout depth rather than frame count. We emphasize---and prove next---that the attention parameterization by itself carries no causal semantics; its causal content is supplied by the interventional training objective.

\subsection{Field Intervention and Counterfactual Rollout}
Following \citet{pearl2009}, we modify local field values inside a spatial mask and roll out with the same operator. Given center $p^*$ and target value $v\in\mathbb{R}^C$,
\begin{equation}
\mathcal{F}_{\mathcal{I}}(p,t)=m(p)\,v+(1-m(p))\,\mathcal{F}(p,t),
\label{eq:field-int}
\end{equation}
with Gaussian mask $m(p)=\exp(-\|p-p^*\|^2/2\sigma^2)$ whose bandwidth is
tied to that of Eq.~\eqref{eq:field-kernel}. Theorem~\ref{thm:field}
derives the exact kernel-column profile of a single coefficient intervention;
the convex blend above approximates that profile up to a smoothness residual
and a finite-grid boundary-normalization term. Intervening in field space
avoids synthesizing photorealistic counterfactual images while preserving
locality and differentiability. We instantiate object removal (learnable null
state), attribute modification (lightweight transformation network), and
custom interventions. After intervention,
$\mathbf{F}^0_{\mathcal{I}}=\mathcal{I}(\mathbf{F}^0)$ and
$\mathbf{F}^k_{\mathcal{I}}=\mathcal{T}^k(\mathbf{F}^0_{\mathcal{I}})$,
$k=1,\dots,T$; factual and counterfactual trajectories share the same
operator and readout (a shared-transition principle), yielding outputs $o$
and $o_{\mathcal{I}}$.

\subsection{Cross-Modal Fusion and Training}
\label{sec:fusion}
The propagated field (query) is fused with language tokens $\mathbf{L}$ (key/value) via $\mathbf{F}_{\mathrm{fused}}=\mathrm{CrossAttn}(\mathbf{F}^T,\mathbf{L},\mathbf{L})$, followed by pooling and an MLP readout. Training proceeds in three stages (field grounding on intervention-rich pretraining data, structure refinement under multi-environment supervision, downstream fine-tuning) with
\begin{equation}
\mathcal{L}=\mathcal{L}_{\mathrm{lm}}+\lambda_1\mathcal{L}_{\mathrm{inv}}+\lambda_2\mathcal{L}_{\mathrm{cf}}+\lambda_3\mathcal{L}_{\mathrm{sp}}+\lambda_4\mathcal{L}_{\mathrm{sm}}.
\label{eq:loss-total}
\end{equation}
$\mathcal{L}_{\mathrm{lm}}$ is standard language modeling. The invariance loss compares propagated fields across environment pairs of the same scene \citep{peters2016}; guided by the analysis below, pairs vary nuisance while preserving the causal parameterization (color jitter and background substitution; viewpoint perturbation is excluded because it changes the field's spatial indexing), and the loss is computed over a foreground mask $B\in\{0,1\}^K$ (broadcast across channels) covering regions unaffected by the nuisance change,
\begin{equation}
\mathcal{L}_{\mathrm{inv}}=\big(\textstyle\sum_k B_k\,C\big)^{-1}\big\|B\odot(\mathbf{F}^T_{e_1}-\mathbf{F}^T_{e_2})\big\|_F^2,
\label{eq:loss-inv}
\end{equation}
so regions whose content legitimately differs (e.g.\ substituted background) are not forced to agree. The counterfactual evolution loss supervises rollout against target post-intervention trajectories,
\begin{equation}
\mathcal{L}_{\mathrm{cf}}=T^{-1}\textstyle\sum_{k=1}^T (KC)^{-1}\|\mathbf{F}^k_{\mathcal{I}}-\mathbf{F}^k_{\mathrm{gt}}\|_F^2,
\label{eq:loss-cf}
\end{equation}
where $\mathbf{F}^k_{\mathrm{gt}}$ encodes the \emph{ground-truth post-intervention observations} of the synthetic pretraining environments (which execute the intervention in the simulator) with an EMA copy of the encoder under stop-gradient; anchoring targets to simulator-produced observations rules out the degenerate solution of ignoring interventions. $\mathcal{L}_{\mathrm{sp}}$ ($\ell_1$ on $\hat{\mathbf{G}}^t$) and $\mathcal{L}_{\mathrm{sm}}$ (neighboring-pair smoothness) regularize structure and field.

\section{Interventional Identifiability of Field Dynamics}
\label{sec:theory}

This section states precisely which structure can be identified from
field-level rollout supervision. We distinguish the dependency graph of the
\emph{complete transition operator} from the internal attention matrix used
to parameterize one of its branches. We first give a population result for
nonlinear dynamics and then specialize it to a linear finite-basis model,
where finite intervention coverage and nonzero-loss guarantees can be made
quantitative. Full proofs are in
Appendices~\ref{sec:nonlinear}--\ref{sec:field}.

\paragraph{Finite-basis field and transition graph.}
Let $\{\varphi_i\}_{i=1}^K$ be linearly independent spatial basis functions
on a domain $\Omega$. For coefficient blocks
$z=(z_1,\ldots,z_K)$, $z_i\in\mathcal Z_i\subseteq\R^C$, define
\begin{equation}
 (\mathcal Bz)(p)=\sum_{i=1}^K\varphi_i(p)z_i,
 \qquad \mathcal H_K:=\operatorname{range}(\mathcal B).
 \label{eq:theory-synthesis}
\end{equation}
Linear independence makes $\mathcal B$ injective, so every field in
$\mathcal H_K$ has unique coefficients. The Gaussian construction of
Eq.~\eqref{eq:field-kernel} is the sampled instance used by InfluenceField.
We assume that the field update closes on $\mathcal H_K$ and define its
coefficient transition by
$\hat f=\mathcal B^{-1}\circ\mathcal T\circ\mathcal B$. On the sampled grid,
this coefficientization is well defined whenever the synthesis matrix is
invertible and each updated grid is reconstructed in the same basis.
Let $z^{t+1}=f^*(z^t)$ on an open connected product domain
$\mathcal Z=\prod_i\mathcal Z_i$, where
$f^*=(f^*_1,\ldots,f^*_K)$ is $C^1$. Its directed
transition-dependency graph is
\begin{equation}
 E^*:=\{(i,j):D_i f^*_j(z)\not\equiv0\text{ on }\mathcal Z\}.
 \label{eq:transition-graph}
\end{equation}
Thus $i\to j$ means that the next state at location $j$ functionally
depends on the current state at location $i$. Because these are time-lagged
dependencies, self-loops and directed cycles are allowed. The graph has a
causal interpretation when $z$ is a sufficient Markov state and the edits
below correspond to valid structural interventions.

Let $y=q(z)$ denote the learned coefficient representation, including the
observation map and encoder. We assume that
$q:\mathcal Z\to\mathcal Y=\prod_i\mathcal Y_i$ is a $C^1$
diffeomorphism onto the product domain $\mathcal Y$ and that corresponding
latent and learned blocks have the same dimension. These are standing
assumptions for the result below, not properties implied by the network
architecture alone.

\paragraph{Target-aligned interventions.}
Fix an anchor $a_i\in\mathcal Z_i$ and let $R_i:\mathcal Z\to\mathcal Z$
replace $z_i$ by $a_i$. The learned edit $Q_i$ may choose its replacement
value as a function of the input, but the map $z\mapsto Q_iq(z)$ is continuous
and changes only learned block $i$. We impose four population conditions.
(N1) The factual source distribution has support $\mathcal Z$, and zero
factual field-state residual gives, for every $z\in\mathcal Z$,
\begin{equation}
 \hat f(q(z))=q(f^*(z));
 \label{eq:nonlinear-factual}
\end{equation}
(N2) for every location $i$, the paired \emph{pre-intervention} source-state
distribution has support $\mathcal Z$, including the original value of
$z_i$; (N3) zero one-step counterfactual \emph{field-state} loss gives
\begin{equation}
 \hat f(Q_iq(z))=q(f^*(R_i z));
 \label{eq:nonlinear-cf}
\end{equation}
for every $z\in\mathcal Z$. Full support and continuity extend zero
population residuals to these pointwise identities. In (N3), the
stop-gradient target encoder in Eq.~\eqref{eq:loss-cf} is idealized as the
same population representation $q$; language-answer loss alone would not
imply this identity. (N4) The transition is one-step separating on the
relevant intervention pairs:
\begin{equation}
 \hat f(Q_iq(z))=\hat f(q(R_i z))
 \ \Longrightarrow\ Q_iq(z)=q(R_i z).
 \label{eq:nonlinear-separation}
\end{equation}
Injectivity of $\hat f$ on
$Q_iq(\mathcal Z)\cup q(R_i\mathcal Z)$ is a simple sufficient condition.
Alternatively, direct alignment of the post-intervention state at step $0$
makes (N4) unnecessary. This separation condition is substantive: a
many-to-one transition can map distinct edited states to the same successor.

Factual fitting alone cannot enforce the desired coordinates. For any
diffeomorphism $M$, the pair $q_M=M\circ q$ and
$\hat f_M=M\circ\hat f\circ M^{-1}$ produces the same factual trajectories;
a coordinate-mixing $M$ can change the transition graph. Target-aligned
interventions remove precisely this ambiguity.

\begin{theorem}[Nonlinear transition-graph identifiability]
\label{thm:nonlinear-ident}
Under the standing finite-basis, product-domain, and regularity assumptions
above and (N1)--(N4), every zero-loss representation is blockwise
identifiable:
there exist $C^1$ diffeomorphisms
$\psi_i:\mathcal Z_i\to\mathcal Y_i$ such that
\begin{equation}
 q(z)=(\psi_1(z_1),\ldots,\psi_K(z_K)).
 \label{eq:block-ident}
\end{equation}
Moreover, for every $z\in\mathcal Z$,
\begin{equation}
 D_i\hat f_j(q(z))=
 D\psi_j(f^*_j(z))\,D_i f^*_j(z)\,
 [D\psi_i(z_i)]^{-1}.
 \label{eq:block-jacobian}
\end{equation}
Consequently, $D_i\hat f_j(q(z))=0$ if and only if
$D_i f^*_j(z)=0$, and the learned and true transition-dependency graphs are
identical. Transition functions and derivative magnitudes remain identified
only up to the within-location reparameterizations $\{\psi_i\}$.
\end{theorem}

\emph{Proof sketch.}
Applying Eq.~\eqref{eq:nonlinear-factual} to $R_i z$ and comparing with
Eq.~\eqref{eq:nonlinear-cf}, left-cancellability yields
$Q_iq(z)=q(R_i z)$. Since $Q_i$ leaves every block $j\neq i$ unchanged,
$q_j(z)=q_j(R_i z)$ for all contexts. Repeating this argument for every
$i\neq j$ shows that $q_j$ depends only on $z_j$, proving
Eq.~\eqref{eq:block-ident}. Differentiating the conjugacy
$\hat f\circ q=q\circ f^*$ then gives Eq.~\eqref{eq:block-jacobian}; its two
outer factors are nonsingular, so they preserve the zero--nonzero status and
rank of every Jacobian block. \hfill$\square$

One anchor value per node is sufficient here because its edit is paired with
every pre-intervention state in the product domain. This is fundamentally
different from observing one finite training pair per node, which is not
sufficient for global nonlinear identification: a smooth coordinate-mixing
diffeomorphism can agree with the desired map on any finite set and differ
elsewhere. Known target alignment removes the usual permutation ambiguity; if
targets are observed only up to an unknown relabeling, the conclusion holds up
to that permutation. This result is tailored to the deterministic finite-basis
field transition and should be viewed within the broader literature on
intervention-based nonlinear causal representation learning
\citep{brehmer2022,lippe2022citris}, rather than as a general nonlinear causal
representation theorem.

The same argument yields exact multi-step rollout for each calibrated
intervention:
$\hat f^k(Q_iq(z))=q((f^*)^k(R_i z))$. An unseen replacement value
$a\in\mathcal Z_i$ is realizable when the learned edit uses the calibrated
value $\psi_i(a)$; graph recovery alone does not assign semantic units to
intervention values.

\begin{proposition}[Finite-loss intervention equivariance]
\label{prop:nonlinear-finite}
Suppose a uniform $\mu>0$ satisfies
$\|\hat f(u)-\hat f(v)\|\ge\mu\|u-v\|$ for every relevant state pair. Define
$r^i_{\mathrm f}(z)=\hat f(q(R_i z))-q(f^*(R_i z))$ and
$r^i_{\mathrm c}(z)=\hat f(Q_iq(z))-q(f^*(R_i z))$. Then
\begin{equation}
 \sum_{j\ne i}\|q_j(z)-q_j(R_i z)\|^2
 \le \frac{2}{\mu^2}\bigl(\|r^i_{\mathrm f}(z)\|^2+
 \|r^i_{\mathrm c}(z)\|^2\bigr).
 \label{eq:nonlinear-finite}
\end{equation}
If the two squared residuals have means at most
$\varepsilon_{\mathrm f,i}$ and $\varepsilon_{\mathrm c,i}$ under the paired
source distribution for node $i$, the root-mean-square off-target change is at
most $\sqrt{2(\varepsilon_{\mathrm f,i}+\varepsilon_{\mathrm c,i})}/\mu$.
This controls approximate intervention equivariance, not nonlinear Jacobian
support; support recovery from finite loss additionally requires
derivative-level control and a minimum edge strength. Direct step-$0$
alignment yields the analogous bound without the inverse-Lipschitz factor.
\end{proposition}

\paragraph{Linear specialization.}
The nonlinear result requires population coverage across contexts. Under the
stronger assumption of linear dynamics and mixing, finitely many experiments
suffice and the guarantees become quantitative. Let scalar state
$Z^t\in\R^K$ evolve as $Z^{t+1}=A^*Z^t$, with edge $i\to j$ defined by
$A^*_{ji}\neq0$. Observations are $X^t=g(Z^t)$ with $g$ injective; the learner
uses $F^t=h(X^t)=HZ^t$, $H\in\GL(K)$, and
$F^{t+1}=\hat A F^t$. This corresponds to the frozen-$G$,
$\alpha=1$, no-refinement limit of Eq.~\eqref{eq:propagation}. We assume:
(A1) the stated linear dynamics; (A2) linear invertible mixing; (A3) factual
initial states span $\R^K$; (A4) every node receives one known-site,
nontrivial hard intervention; (A5) $A^*$ is nonsingular; and (A6) the factual
and corresponding one-step counterfactual population losses are zero. These
assumptions characterize the surrogate and are not consequences of
$\mathcal L_{\mathrm{inv}}$ or $\mathcal L_{\mathrm{lm}}$.

\begin{lemma}[Observational non-identifiability]
\label{lem:nonident}
Under factual fitting alone, the zero-loss solutions form the similarity
orbit $\{(H',H'A^*H'^{-1}):H'\in\GL(K)\}$. Consequently, transition support,
sparsity, and directionality are not identifiable in general. In particular,
if $A^*$ is diagonalizable over $\R$, an observationally equivalent solution
is diagonal and hence symmetric.
\end{lemma}

Hence an attention-based influence matrix---however sparse or asymmetric---carries no causal guarantee by itself; interventional supervision supplies exactly the missing constraint:

\begin{theorem}[Linear identification from finite interventions]
\label{thm:ident}
Under (A1)--(A6), $H$ is diagonal and
$\hat A_{ij}=(h_i/h_j)A^*_{ij}$. Hence
$\supp(\hat A)=\supp(A^*)$, including edge directions, while weights are
identified up to node-wise scaling. Every learned coordinate intervention
has an exact latent coordinate counterpart, and its subsequent rollout is
exact. A specified unseen latent target additionally requires calibration of
its representation-space value.
\end{theorem}

\emph{Proof sketch.} Factual optimality and (A3) give $\hat A=HA^*H^{-1}$. Zero one-step counterfactual loss for a covered intervention on node $i$ forces the learner's field-space edit to map, through $H^{-1}$, onto the latent edit: the direction $e_i$ acts as a probe that the learner can match only if $e_i$ is an eigenvector of $H$. Coverage of every node diagonalizes $H$. \hfill$\square$

The scaling cancels in field space
($\hat A^k\tilde F^0=HA^{*k}\tilde Z^0$), so it is latent unit freedom rather
than rollout error. With unknown \emph{site correspondence}, the remaining
ambiguity is a permutation followed by node-wise scaling
\citep{squires2023,ahuja2023}. Under partial coverage, predictions for
covered-node edits remain exact, but the covered subgraph is determined only
when no covered node feeds an uncovered one. Appendix~\ref{sec:partial} gives
the exact solution set and shows that generic full-graph identification
requires full coverage.

\begin{proposition}[Finite-loss stability of the linear specialization]
\label{prop:approx}
Assume (A1)--(A5), with (A3) strengthened quantitatively (training initial-state second moment bounded below), intervention states and magnitudes bounded, and the encoder scale normalized to $\|H\|_2=1$. In place of (A6), suppose the factual and per-intervention \emph{one-step} counterfactual population losses are bounded by $\varepsilon\leq\varepsilon_0$ (an explicit threshold), and every covered intervention has magnitude $|v-Z^0_i|\geq\delta>0$. Then there is a diagonal $D$ with
\begin{equation}
\big\|D^{-1}\hat A D-A^*\big\|_F\;\leq\;C\,\kappa(H)^{3}\,\frac{\sqrt{\varepsilon}}{\delta},
\label{eq:approx-bound}
\end{equation}
where $C$ is explicit, depending on $K$, the training-state conditioning, and the norms of $A^*$, but not on $\varepsilon$, $\delta$, or $H$. If moreover the smallest nonzero $|A^*_{ij}|$ exceeds twice the bound, entrywise thresholding of $D^{-1}\hat A D$ recovers $\mathrm{supp}(A^*)$ exactly.
\end{proposition}

Because Eq.~\eqref{eq:loss-cf} averages over rollout depth, intervention
sites, field entries, and examples, its scalar value is not itself the
$\varepsilon$ in Proposition~\ref{prop:approx}. We measure the required
one-step factual and maximum per-site counterfactual errors separately.
The sufficient threshold $\varepsilon_0$ is conservative and is not reached
by our checkpoints. Figure~\ref{fig:ident}(b) therefore tests the predicted
$\sqrt{\varepsilon}$ functional form outside the certified regime rather than
claiming empirical verification of the bound. Under field-structured mixing,
$H=W_\sigma\otimes\Phi$ and
$\kappa(H)=\kappa(W_\sigma)\kappa(\Phi)$, so the certified upper bound worsens
polynomially with field conditioning; empirical monotonicity is a separate
observation.

\begin{theorem}[Finite-basis intervention geometry]
\label{thm:field}
Let the field be sampled on the patch grid and let
$W_\sigma\in\R^{K\times K}$ be the known, invertible Gaussian interpolation
matrix. In the linear shared-channel model
$F=W_\sigma Z\Phi^{\!\top}$, $\Phi\in\GL(C)$, with nonsingular location
dynamics $Z^{t+1}=A^*Z^t$, assume $K\ge2$ and that the collection of factual
coefficient-state columns spans $\R^K$. Then (i) factual equality implies
$A^*=W_\sigma^{-1}\hat A W_\sigma$; (ii) replacing coefficient block $i$
changes the sampled field along the unique profile $W_\sigma e_i$, whereas a
hard replacement of sampled row $i$ is not a single-coefficient intervention
and has a positive one-step loss floor for every nonzero coefficient edit;
and (iii) if \emph{known latent
channel offsets} spanning $\R^C$ and their learned offsets are paired at one
site, the shared map $\Phi$ is identified and the calibration transfers to
all sites.
\end{theorem}

Claim (iii) is a shared-parameter calibration result, not a statement that one
ordinary intervention suffices in the deployed $C{=}256$ model. The convex
blend in Eq.~\eqref{eq:field-int} approximates the ideal kernel-column edit:
it contains a smoothness residual, and finite-grid row normalization creates
a boundary mismatch (Appendix~\ref{sec:field}). A centered additive edit has
the exact profile. The principle is kernel matching rather than Gaussianity.

\paragraph{What is identified in the deployed parameterization.}
For Eq.~\eqref{eq:propagation}, write $\bar G(y)=\hat G^t$ and let $r$ denote
the refinement branch. The block Jacobian of the complete transition is
\begin{equation}
\begin{aligned}
 D_i\hat f_j(y)={}&\alpha\,\bar G_{ij}(y)I_C
 +(1-\alpha)D_i r_j(y)\\
 &+\alpha\sum_{p=1}^K
 y_p\bigl(\nabla_{y_i}\bar G_{pj}(y)\bigr)^{\!\top}.
\end{aligned}
 \label{eq:full-transition-jacobian}
\end{equation}
Theorem~\ref{thm:nonlinear-ident} therefore identifies the dependency graph
of the complete transition $\hat f$, not the support of raw $\bar G$. The
latter is an influence parameterization whose agreement with operational
interventions is evaluated empirically; it coincides with the transition
matrix only up to orientation: in the frozen-$G$, $\alpha{=}1$,
no-refinement linearization, $\hat A=\bar G^{\top}$.

\paragraph{Scope.}
The nonlinear result is a population statement on the fixed
finite-dimensional space $\mathcal H_K$; it does not identify an unrestricted
continuous operator, neural-network parameters, or raw attention weights.
Finite isolated interventions cannot establish global nonlinear
identifiability without additional function-class restrictions. The linear
results provide finite-coverage and finite-loss refinements but rely on a
stronger surrogate. The theory-aligned control in Table~\ref{tab:controls}
and the operational-influence comparison in
Section~\ref{sec:empirical-ident} measure parts of this theory--architecture
gap; neither is a proof that the nonlinear network satisfies the population
assumptions.

\section{Experiments}
\label{sec:exp}
We address five questions. \textbf{Q1:} Does InfluenceField improve causal video reasoning? \textbf{Q2:} Are the gains due to causal objectives rather than capacity, and does the model use directional influence? \textbf{Q3:} Are the theory-motivated recovery and stability trends observed empirically? \textbf{Q4:} Does the learned field remain robust under distribution shift and intervention? \textbf{Q5:} Does it recover causal structure on discovery benchmarks?

\subsection{Setup}
\label{sec:implementation}
\paragraph{Datasets and baselines.}
We pretrain on CLEVRER \citep{yi2019clevrer} and Causal3DIdent \citep{vonkugelgen2021}/CITRIS \citep{lippe2022citris} for intervention-rich grounding, then fine-tune downstream. For QA, Qwen3-VL-8B \citep{bai2025qwen3vl} is the partially frozen backbone, compared against GPT-4o \citep{openai2024gpt4o}, Gemini 2.5 Flash \citep{comanici2025gemini25}, InternVL2.5 \citep{chen2024internvl}, LLaVA-OneVision \citep{li2025llavaov}, Perception-LM \citep{cho2025perceptionlm}, and Qwen2.5-VL \citep{bai2025qwen25vl}. For graph-structure tasks, backbone MLLMs (GPT-4o, Gemini 2.0 \citep{deepmind2024gemini20}, LLaMA~4 \citep{meta2025llama4}, Grok-2v \citep{xai2024grok2}) serve as frozen feature providers, compared against META/COAT \citep{liu2024coat} and MLLM-CD \citep{li2025mllmcd}; only InfluenceField and light heads are trained on top of frozen features.
\textbf{Benchmarks:} CausalVQA \citep{foss2025causalvqa} contains
793 QA pairs spanning Anticipation, Counterfactual, Descriptive,
Hypothetical, and Planning; \emph{Reasoning} pools the four
non-descriptive categories. We additionally evaluate zero-shot transfer on
the NExT-QA validation set \citep{xiao2021nextqa}. \textbf{Discovery:} We use
MAG and Lung Cancer following MLLM-CD. Unless otherwise stated, downstream
results summarize three runs and synthetic-identifiability results summarize
five; uncertainties are reported in the corresponding table cells or
captions.

\paragraph{Implementation details.}
Patch features are projected onto a $14\times14$ field with $C{=}256$
channels and $4$ attention heads. Unless noted, $T{=}3$ and $\sigma{=}1.0$;
the intervention mask uses the same bandwidth, motivated by the ideal
kernel-column profile in Theorem~\ref{thm:field} but subject to the residuals
described in Appendix~\ref{sec:field}. Loss weights are in
Appendix~\ref{app:hyper}. The synthetic experiments in
Fig.~\ref{fig:ident} probe the linear specialization through coverage,
finite-loss trends, and kernel conditioning; they do not establish the
population assumptions of Theorem~\ref{thm:nonlinear-ident}.

\paragraph{Protocol details.}
\emph{OOD splits.} We evaluate four controlled shifts derived from the
CausalVQA validation pool: \textbf{Scene Shift} ($n{=}214$) varies
background and viewpoint configurations; \textbf{Composition Shift}
($n{=}189$) evaluates novel object--action combinations; \textbf{Template
Shift} ($n{=}793$) uses held-out question templates; and
\textbf{Intervention Shift} ($n{=}176$) changes
intervention-type--target combinations. Exact construction rules and split
manifests are provided in Appendix~\ref{app:ood}. At the observed
accuracies, the per-split binomial standard errors are approximately
1.7--3.8 percentage points.
\emph{Coherence metrics.} Each counterfactual question is paired with its factual counterpart on the same video; annotators mark whether answers must agree, whether and to what the answer must change, and the admissible outcome set. \textbf{Consistency} is the fraction of pairs satisfying the annotated agreement relation; \textbf{Causal Flip Accuracy} is, among answer-changing pairs, the fraction changing to the annotated answer; \textbf{Invalid Transition Rate} is the fraction of counterfactual answers outside the admissible set.

\begin{table}[!htbp]
\centering
\small
\setlength{\tabcolsep}{2.0pt}
\begin{tabular}{lccccccc}
\toprule
Method & Ant. & Ctrf. & Desc. & Hypo. & Plan. & Reas. & All \\
\midrule
GPT-4o & 37.0 & 47.2 & 64.6 & 34.3 & 53.9 & 43.2 & 51.0 \\
Gemini 2.5 Flash & 45.2 & 49.5 & 76.0 & 52.2 & 68.8 & 53.5 & 61.7 \\
InternVL2.5 & 42.8 & 43.8 & 48.3 & 40.3 & 58.8 & 47.1 & 47.5 \\
LLaVA-OneVision & 33.7 & 46.1 & 55.2 & 23.9 & 51.8 & 39.6 & 45.3 \\
Perception-LM & 41.8 & 48.3 & 51.7 & 46.3 & 61.7 & 49.1 & 50.1 \\
Qwen2.5-VL & 35.6 & 46.1 & 61.8 & 40.3 & 48.9 & 41.8 & 49.1 \\
Qwen3-VL-8B (base) & 40.5 & 49.2 & 66.4 & 46.9 & 56.7 & 47.4 & 54.3 \\
\textbf{InfluenceField} & \textbf{52.5} & \textbf{61.0} & \textbf{79.0} & \textbf{60.5} & \textbf{73.2} & \textbf{60.8} & \textbf{67.4} \\
\midrule
Human & 83.5 & 76.6 & 90.6 & 75.9 & 84.0 & 81.4 & 84.7 \\
\bottomrule
\end{tabular}
\caption{Causal video QA on CausalVQA (\%). Best automated results in bold. Automated results are means over three seeds; per-cell standard deviations range from 0.36 to 1.04.}
\label{tab:main}
\end{table}

\subsection{Q1: Causal Video Reasoning}
InfluenceField reaches 67.4 overall (Table~\ref{tab:main}), 13.1 points above its backbone, and outperforms every automated baseline including Gemini~2.5 Flash. The largest gains are on Planning ($56.7{\to}73.2$) and Hypothetical ($46.9{\to}60.5$), the two categories that most directly require action--consequence reasoning; aggregate Reasoning rises from 47.4 to 60.8. On zero-shot NExT-QA (Table~\ref{tab:nextqa}), which the model never sees during training, the gain concentrates on causal questions (Causal $+4.4$, Temporal $+1.4$, Descriptive $+0.4$), indicating that the mechanism transfers rather than that the model exploits CausalVQA-specific regularities.

\begin{table}[!htbp]
\centering
\small
\setlength{\tabcolsep}{4pt}
\begin{tabular}{lcccc}
\toprule
Method & Causal & Temporal & Desc. & All \\
\midrule
Qwen3-VL-8B (base) & 71.4 & 68.2 & 80.5 & 72.3 \\
\textbf{InfluenceField} & \textbf{75.8} & \textbf{69.6} & 80.9 & \textbf{74.9} \\
\bottomrule
\end{tabular}
\caption{Zero-shot accuracy on NExT-QA validation by question type; published baselines are quoted in Appendix~\ref{app:pubnextqa}. Results are means over three seeds; per-cell standard deviations range from 0.27 to 0.63.}
\label{tab:nextqa}
\end{table}

\subsection{Q2: Capacity-Matched and Structure Controls}
\label{sec:controls}
The component ablation (Table~\ref{tab:ablation}) shows the field and propagation alone contribute substantially; Table~\ref{tab:controls} isolates the causal contribution. \textbf{C1} replaces the propagation module with a parameter/FLOP-matched transformer block trained with $\mathcal{L}_{\mathrm{lm}}$ only. \textbf{C2} keeps the architecture but removes $\mathcal{L}_{\mathrm{inv}}/\mathcal{L}_{\mathrm{cf}}$, matching compute and feeding invariance-pair augmentations as ordinary augmentation. \textbf{C3} destroys structure at inference: symmetrized $(\hat G{+}\hat G^{\!\top})/2$, transposed $\hat G^{\!\top}$, or frozen random $\hat G$ (evaluable on OOD splits since inference-time). \textbf{C4} trains with shuffled intervention--outcome pairings. We also report the theory-aligned configurations of the Scope paragraph.

\begin{table*}[!htbp]
\centering
\small
\setlength{\tabcolsep}{3.5pt}
\begin{tabular}{lccccccc}
\toprule
Variant & All & Desc. & Plan. & Ctrf. & Drop$\downarrow$ & Cons.$\uparrow$ & Flip$\uparrow$ \\
\midrule
Full model & \resultpmc{67.4}{0.49} & \resultpmc{79.0}{0.91} & \resultpmc{73.2}{0.51} & \resultpmc{61.0}{0.98} & \resultpmc{5.9}{0.63} & \resultpmc{78.6}{0.49} & \resultpmc{59.7}{0.53} \\
\midrule
C1 matched Transf. & \resultpmc{59.1}{0.35} & \resultpmc{74.0}{0.58} & \resultpmc{62.5}{1.04} & \resultpmc{52.4}{0.38} & \resultpmc{7.2}{0.32} & \resultpmc{68.0}{0.47} & \resultpmc{47.2}{0.44} \\
C2 no causal losses & \resultpmc{62.9}{0.41} & \resultpmc{72.6}{0.60} & \resultpmc{67.5}{0.43} & \resultpmc{55.6}{0.47} & \resultpmc{7.3}{0.36} & \resultpmc{69.5}{1.00} & \resultpmc{49.1}{0.71} \\
\midrule
C3 symmetrized & \resultpmc{63.8}{0.30} & \resultpmc{78.2}{0.40} & \resultpmc{67.3}{1.04} & \resultpmc{55.5}{1.01} & \resultpmc{7.0}{0.71} & \resultpmc{72.4}{0.78} & \resultpmc{51.3}{1.04} \\
C3 transposed & \resultpmc{59.3}{0.33} & \resultpmc{77.2}{0.68} & \resultpmc{62.1}{0.63} & \resultpmc{52.7}{0.63} & \resultpmc{7.6}{0.47} & \resultpmc{70.3}{0.44} & \resultpmc{47.8}{1.02} \\
C3 random frozen & \resultpmc{57.8}{0.73} & \resultpmc{75.0}{0.42} & \resultpmc{60.2}{0.96} & \resultpmc{50.9}{0.36} & \resultpmc{8.3}{0.68} & \resultpmc{66.4}{1.04} & \resultpmc{45.1}{0.37} \\
\midrule
C4 shuffled CF & \resultpmc{64.0}{0.47} & \resultpmc{78.3}{0.35} & \resultpmc{69.7}{0.57} & \resultpmc{56.3}{0.89} & \resultpmc{7.2}{0.71} & \resultpmc{70.4}{0.67} & \resultpmc{49.3}{0.72} \\
\midrule
Frozen-$\hat G$ rollout & \resultpmc{67.1}{0.76} & \resultpmc{78.8}{0.42} & \resultpmc{72.7}{0.34} & \resultpmc{60.5}{0.49} & \resultpmc{6.1}{0.66} & \resultpmc{78.0}{1.01} & \resultpmc{59.0}{0.81} \\
$\alpha{=}1$ (no refine) & \resultpmc{66.3}{0.91} & \resultpmc{78.1}{0.35} & \resultpmc{71.9}{0.55} & \resultpmc{59.7}{0.48} & \resultpmc{6.2}{0.32} & \resultpmc{77.1}{0.74} & \resultpmc{58.3}{0.44} \\
$\alpha{=}1$ + frozen-$\hat G$ & \resultpmc{65.9}{0.46} & \resultpmc{77.9}{0.91} & \resultpmc{71.5}{0.42} & \resultpmc{59.2}{0.73} & \resultpmc{6.3}{0.33} & \resultpmc{76.6}{0.49} & \resultpmc{57.7}{0.87} \\
\bottomrule
\end{tabular}
\caption{Capacity-matched (C1/C2), structure-destruction (C3),
shuffled-supervision (C4), and theory-aligned controls. Entries are mean
$\pm$ standard deviation over three runs. Drop is computed per run as the
in-domain All accuracy minus the unweighted mean over the four OOD splits;
the corresponding per-split results are reported in
Table~\ref{tab:ctrl-ood}.}
\label{tab:controls}
\end{table*}

The controls address a natural concern: that the model has learned a useful \emph{relational} representation rather than structure with interventional content. First, the +13.1 in-domain gain decomposes as follows. A parameter-matched transformer block (C1) accounts for +4.8; the field-and-propagation architecture adds +2.3 (Table~\ref{tab:ablation}, DIF+DP: 61.4); the causal objectives contribute the remaining +6.0 (+4.5 relative to C2, which additionally sees the invariance-pair data as augmentation). More telling is what the non-causal variants fail to recover: C2 regains only about a fifth of the Consistency and Flip improvements (69.5 and 49.1, against 66.8/46.8 for the backbone and 78.6/59.7 for the full model), and its OOD drop (7.3) equals the backbone's. Robustness and coherence therefore come from the causal objectives specifically. Second, destroying edge \emph{directionality} at inference (C3-transposed) degrades the intervention-sensitive categories (Planning $-11.1$, Counterfactual $-8.3$) while leaving Descriptive nearly intact ($-1.8$), and accuracy falls monotonically from full to symmetrized to transposed to random $\hat G$. The model therefore uses the specific learned directions---the empirical counterpart of Lemma~\ref{lem:nonident}; whether those directions are also correct is tested in Section~\ref{sec:empirical-ident}. Third, shuffling the intervention--outcome pairings (C4) lowers Counterfactual accuracy to slightly below the no-$\mathcal{L}_{\mathrm{cf}}$ level (56.3 vs.\ 59.1 in Table~\ref{tab:ablation}) and returns Flip to the no-causal-loss level (49.3 vs.\ 49.1 for C2): $\mathcal{L}_{\mathrm{cf}}$ works through the \emph{content} of the pairs, not as a generic regularizer. Finally, enforcing the linear-theory-aligned restrictions
($\alpha{=}1$ and rollout-constant $\hat G$) reduces overall accuracy by
1.5 points ($67.4{\to}65.9$). This control measures the empirical cost of
narrowing the deployed architecture toward the setting analyzed by the
linear theory.

\begin{figure}[!htbp]
\centering
\includegraphics[width=\columnwidth]{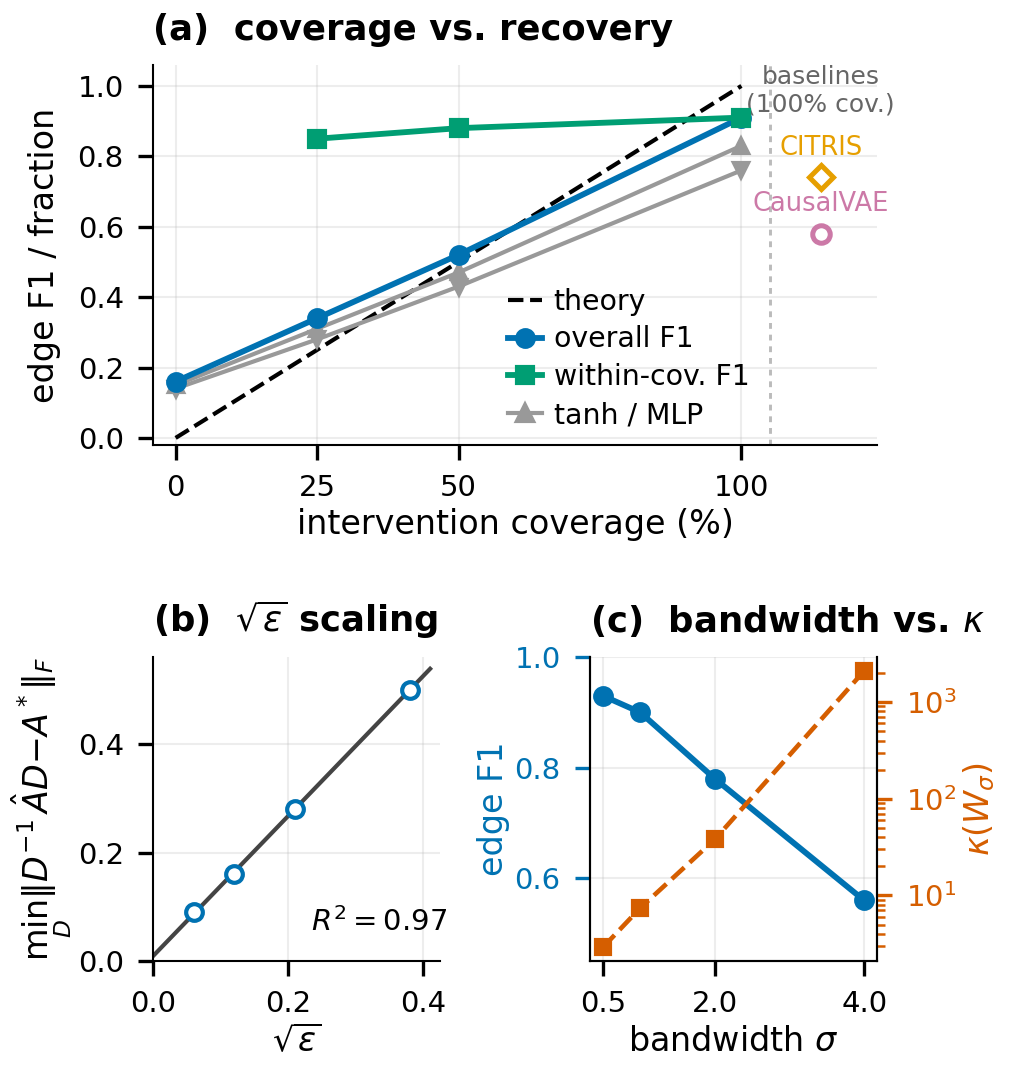}
\caption{Identifiability on a synthetic simulator with known $A^*$ (five seeds; protocols in Appendix~\ref{app:ident-protocol}). \textbf{(a)} Coverage vs.\ recovery (Thm.~\ref{thm:ident} and its partial-coverage refinement): the dashed line is the theory-predicted identifiable-pair fraction and refers to the \emph{overall} curve; within-covered F1 counts only pairs inside the covered set. \textbf{(b)} Empirical finite-loss $\sqrt{\varepsilon}$ trend (Prop.~\ref{prop:approx}) over early-stopping checkpoints. \textbf{(c)} Recovery vs.\ bandwidth as $\kappa(W_\sigma)$ grows (Thm.~\ref{thm:field}).}
\label{fig:ident}
\end{figure}

\subsection{Q3: Empirical Identifiability}
\label{sec:empirical-ident}
On a synthetic latent simulator with known $A^*$ (grid $K{=}64$, linear
dynamics, rendered observations), which satisfies (A1) by construction and
(A2) up to encoder training error,
Fig.~\ref{fig:ident} provides controlled diagnostics against the linear
predictions; panel (a) additionally probes robustness beyond the assumptions
with two nonlinear-dynamics variants, and adapts two causal-representation
baselines to the simulator with the same intervention information (protocol in
Appendix~\ref{app:ident-protocol}).

The observed trends are consistent with the linear analysis, subject to its
stated scope. (a) Recovery tracks coverage as Theorem~\ref{thm:ident} and its
partial-coverage refinement (Appendix~\ref{sec:partial}) predict:
near-chance at zero coverage, high F1 within the covered subgraph at partial
coverage, and near-exact recovery at full coverage. The tanh and MLP variants
are robustness stress tests rather than verifications of
Theorem~\ref{thm:nonlinear-ident}; in particular, the unconstrained MLP does
not have a globally fixed Jacobian support. Cross-seed correlation of
$\hat G$ is 0.91 with $\mathcal{L}_{\mathrm{cf}}$ versus 0.43 without,
consistent with reduced variability along observationally equivalent
solutions. The learned $\hat G$ also agrees with operational influence
(intervening at site $i$ and recording $\|\Delta F^1_j\|$), with Spearman
$\rho{=}0.74$ and 82\% direction agreement. (b) Across five seed-specific
early-stopping trajectories (20 checkpoints per seed; $N{=}100$), aligned
recovery error follows an empirical
$\sqrt{\varepsilon}$ trend while $\kappa(H)$ remains approximately constant.
Because these checkpoints lie outside the certified
$\varepsilon_0$ regime, this supports the functional form but does not verify
Proposition~\ref{prop:approx}. (c) Recovery degrades as
$\kappa(W_\sigma)$ grows. The optimized hard-row loss is $0.061$, close to
its exact minimum of $0.059$, whereas the kernel-profile edit attains
$0.007$, consistent with
Theorem~\ref{thm:field}(ii). Although $\sigma{=}0.5$ gives marginally better
synthetic recovery, it reduces downstream QA by 0.9 points
(Appendix~\ref{app:sensitivity}, Table~\ref{tab:sensitivity}, panel (a)); we
therefore use $\sigma{=}1.0$.

\subsection{Q4: Robustness and Consistency}

\begin{table}[!htbp]
\centering
\small
\setlength{\tabcolsep}{1.8pt}
\begin{tabular}{lccccccc}
\toprule
Method & In-d. & Scene & Comp. & Templ. & Interv. & OOD & Drop$\downarrow$ \\
\midrule
GPT-4o & 51.0 & 45.1 & 44.6 & 42.3 & 41.6 & 43.4 & 7.6 \\
Gemini 2.5 Fl. & 61.7 & 57.5 & 53.9 & 55.8 & 52.6 & 55.0 & 6.8 \\
Qwen3-VL-8B & 54.3 & 48.9 & 47.4 & 47.5 & 44.0 & 47.0 & 7.3 \\
\textbf{Ours} & \textbf{67.4} & \textbf{62.8} & \textbf{61.6} & \textbf{61.9} & \textbf{59.6} & \textbf{61.5} & \textbf{5.9} \\
\bottomrule
\end{tabular}
\caption{OOD generalization (abridged; full table incl.\ InternVL2.5, LLaVA-OneVision, Perception-LM, Qwen2.5-VL, with per-split standard deviations, in Appendix~\ref{app:ood}, Table~\ref{tab:ood-full}).}
\label{tab:ood}
\end{table}

\begin{table}[!htbp]
\centering
\small
\setlength{\tabcolsep}{4pt}
\begin{tabular}{lcccc}
\toprule
Method & Ctrf.\ Acc. & Cons.$\uparrow$ & Inv.Tr.$\downarrow$ & Flip$\uparrow$ \\
\midrule
GPT-4o & 47.2 & 61.8 & 14.9 & 42.7 \\
Gemini 2.5 Flash & 49.5 & 65.1 & 13.5 & 45.6 \\
Qwen2.5-VL & 46.1 & 59.4 & 15.8 & 41.3 \\
Qwen3-VL-8B & 49.2 & 66.8 & 12.9 & 46.8 \\
\textbf{Ours} & \textbf{61.0} & \textbf{78.6} & \textbf{8.1} & \textbf{59.7} \\
\bottomrule
\end{tabular}
\caption{Factual--counterfactual consistency. Results are means over three seeds; per-cell standard deviations range from 0.41 to 1.01.}
\label{tab:consist}
\end{table}

InfluenceField attains the highest mean OOD accuracy (61.5) and the
smallest average degradation (5.9 points) across the four evaluation shifts
(Table~\ref{tab:ood}). The degradation is smallest under Scene Shift (4.6
points) and largest under Intervention Shift (7.8 points), indicating that
generalization to altered interventions remains the most difficult axis.
Because the four splits differ in both size and construction, we treat small
between-split differences descriptively and use the unweighted mean across
shifts as the primary robustness summary. On coherence
(Table~\ref{tab:consist}), Consistency rises $66.8{\to}78.6$, Invalid
Transitions fall $12.9{\to}8.1$, and Flip Accuracy reaches 59.7; control C4
attributes these gains chiefly to the \emph{content} of counterfactual
rollout supervision.

\paragraph{Failure analysis.}
Residual errors on intervention-sensitive questions concentrate in three regimes: interactions involving occluded or off-screen objects, dependencies whose effective horizon exceeds the propagation depth $T$, and spatially ambiguous intervention targets, jointly accounting for 61\% of remaining errors; qualitative cases and a per-regime breakdown are in Appendix~\ref{app:failure}.

\subsection{Q5: Multimodal Causal Discovery}
Averaged over four backbone MLLMs on Lung Cancer, InfluenceField reaches NF
$0.94$ and AF $0.73$, compared with $0.90$/$0.68$ for our reproduction of
MLLM-CD, while reducing ESHD from $5.50$ to $4.30$; on MAG, NF improves
$0.89{\to}0.93$, AF $0.57{\to}0.64$, and ESHD
$13.42{\to}11.10$ (full per-backbone tables with standard deviations in
Appendix~\ref{app:discovery}). These benchmarks are small-scale and their
factors are not spatially indexed, so we use them as structure-oriented
validation complementary to---not covered by---the spatial-node guarantees
of Section~\ref{sec:theory}.

\subsection{Ablation and Efficiency}
\begin{table}[!htbp]
\centering
\small
\setlength{\tabcolsep}{2.0pt}
\begin{tabular}{ccccccccccc}
\toprule
DIF & DP & CI & CF & Ant. & Ctrf. & Desc. & Hypo. & Plan. & Reas. & All \\
\midrule
 & & & & 40.5 & 49.2 & 66.4 & 46.9 & 56.7 & 47.4 & 54.3 \\
\checkmark & & & & 44.1 & 51.8 & 70.1 & 50.6 & 61.2 & 51.1 & 58.0 \\
\checkmark & \checkmark & & & 47.8 & 55.4 & 72.0 & 54.8 & 66.9 & 55.4 & 61.4 \\
\checkmark & \checkmark & \checkmark & & 50.1 & 59.1 & 75.0 & 57.2 & 70.1 & 58.2 & 64.3 \\
\checkmark & \checkmark & \checkmark & \checkmark & \textbf{52.5} & \textbf{61.0} & \textbf{79.0} & \textbf{60.5} & \textbf{73.2} & \textbf{60.8} & \textbf{67.4} \\
\bottomrule
\end{tabular}
\caption{Component ablation (backbone Qwen3-VL-8B). DIF = differentiable influence field; DP = dynamic propagation; CI = invariance; CF = counterfactual supervision. Results are means over three seeds; per-cell standard deviations range from 0.35 to 1.01.}
\label{tab:ablation}
\end{table}

Table~\ref{tab:ablation} shows that DIF and DP provide the representational
and mechanistic foundation ($54.3{\to}61.4$), while CI and CF add the causal
constraints ($+2.9$ and $+3.1$). Accuracy saturates at $T{=}3$ (Planning
$66.4{\to}73.2$ from $T{=}0$) and dips slightly at $T{=}4$; under matched
profiling, the latency overhead over the backbone is 7\%
(Appendix~\ref{app:sensitivity}).

\section{Conclusion}
InfluenceField equips MLLMs with a continuous differentiable field that
propagates directed influence and admits field-level intervention. Our
analysis separates two regimes: a nonlinear population result identifies the
dependency graph of the complete finite-basis transition exactly under the
stated conditions, while the representation remains determined only up to
within-location diffeomorphisms. A linear specialization supplies
partial-coverage and finite-loss guarantees and characterizes the
kernel-column intervention profile. Empirically, the objectives improve
causal reasoning, robustness, and factual--counterfactual coherence, with
controls attributing the gains to the causal objectives specifically.
\textbf{Limitations:} the nonlinear theorem assumes exact full-support
population consistency and one-step transition separation; it identifies the complete transition
graph rather than raw attention weights. The quantitative finite-loss and
coverage results remain linear, field-level edits approximate physical
interventions, and dense propagation is quadratic in the number of
locations. Future work includes finite-sample nonlinear guarantees,
Jacobian-calibrated influence parameterizations, locally banded fields, and
embodied settings.

\clearpage
\appendix

\section*{Technical Appendix}
\noindent This appendix provides complete proofs for the identifiability
results stated in Section~\ref{sec:theory}. Appendix~\ref{sec:setup} fixes
notation and formalizes the losses. Appendix~\ref{sec:nonlinear} proves the
nonlinear population result (Theorem~\ref{thm:nonlinear-ident}) and its
finite-loss equivariance bound. Appendix~\ref{sec:nonident} proves
linear observational non-identifiability (Lemma~\ref{lem:nonident}).
Appendix~\ref{sec:ident} proves exact identifiability from full interventional
coverage (Theorem~\ref{thm:ident}). Appendix~\ref{sec:partial} gives the exact
characterization of partial coverage. Appendix~\ref{sec:approx} proves the
finite-loss result (Proposition~\ref{prop:approx}) with explicit constants and
an explicit regime of validity. Appendix~\ref{sec:field} proves the
field-structured mixing results (Theorem~\ref{thm:field}).
Appendix~\ref{sec:protocols} gives protocols and complete results,
Appendix~\ref{sec:impl} implementation details, Appendix~\ref{sec:addresults}
additional results, and Appendix~\ref{sec:discussion} further discussion.

\section{Notation and Formal Setup}
\label{sec:setup}

\paragraph{Model.} The latent state $Z^t\in\R^K$ evolves as $Z^{t+1}=A^*Z^t$
with $A^*\in\R^{K\times K}$; $\supp(A^*)$ is the true directed graph (cycles
allowed). Observations are $X^t=g(Z^t)$ with $g$ injective. The learner encodes
$F^t=h(X^t)$ and rolls out $F^{t+1}=\Ah F^t$. The assumptions below instantiate
the linear interventional causal-representation setting studied by
\citet{squires2023}, \citet{ahuja2023} and \citet{buchholz2023}, specialized to
a temporal transition operator as in \citet{lippe2022citris}.

\paragraph{Assumptions.}
\begin{itemize}\itemsep0pt
\item[(A1)] Linear latent dynamics $Z^{t+1}=A^*Z^t$.
\item[(A2)] The composite encoding is linear and invertible: $F^t=HZ^t$ for
some unknown $H\in\GL(K)$.
\item[(A3)] The training initial-state distribution spans $\R^K$:
$\Sigma_0:=\E[Z^0(Z^0)^\top]\succ 0$.
\item[(A4)] For each node $i\in[K]$ there is an intervention experiment
$\mathcal{E}_i=(Z^{0,i},v_i)$: the environment applies $\mathrm{do}(Z_i=v_i)$
at step $0$ to initial state $Z^{0,i}$, with $v_i\neq Z^{0,i}_i$, and rolls out
under $A^*$. The learner knows the \emph{site correspondence} and forms
$\tilde F^{0,i}=F^{0,i}+(w_i-F^{0,i}_i)e_i$ for a scalar $w_i$ of its choice.
\item[(A5)] $A^*$ is nonsingular.
\item[(A6)] Zero population loss (exact-recovery sections only).
\end{itemize}

\paragraph{Losses.} Write $\tilde Z^{0,i}=Z^{0,i}+(v_i-Z^{0,i}_i)e_i$ and
$\tilde Z^{k,i}=A^{*k}\tilde Z^{0,i}$. The factual one-step population loss and
the per-node one-step counterfactual losses are
\begin{equation}
\begin{aligned}
\mathcal{L}_{\mathrm{f}}&:=\E\,\norm{\Ah HZ^0-HA^*Z^0}_2^2,\\
\mathcal{L}_{\mathrm{c},i}&:=\norm{\Ah\tilde F^{0,i}-H A^*\tilde Z^{0,i}}_2^2 .
\end{aligned}
\label{eq:losses}
\end{equation}
Throughout, $\norm{\cdot}_2$ on matrices is the spectral norm,
$\norm{\cdot}_F$ the Frobenius norm, $\sigma_{\min}(\cdot)$ the smallest
singular value, and $\kappa(\cdot)$ the condition number.

\medskip
\noindent\textit{Relation to the trained objective.}
Equation~\eqref{eq:loss-cf} averages over rollout depth and examples and
normalizes each field error by $KC$; the training loader may additionally
average nonuniformly over intervention sites. The losses in
Eq.~\eqref{eq:losses} are unnormalized one-step population errors, and the
finite-loss theorem requires their maximum over sites. Consequently, no
dimension-free inequality identifies its $\varepsilon$ with the reported
scalar $\mathcal L_{\mathrm{cf}}$. In
Appendix~\ref{app:ident-protocol} we instead measure, with one fixed
normalization used on both sides of the analysis,
$\varepsilon=\max\{\mathcal{L}_{\mathrm{f}},
\max_i\mathcal{L}_{\mathrm{c},i}\}$ directly.
\par

\medskip
\noindent\textit{$\mathcal{L}_{\mathrm{f}}$ has no direct counterpart in Eq.~\eqref{eq:loss-total}.}
Proposition~\ref{prop:approx-explicit} assumes a bound on the one-step
\emph{factual} field
loss $\E\norm{\Ah F^0-F^1}^2$. The trained objective of Eq.~\eqref{eq:loss-total}
contains no term that penalizes this quantity directly; it is controlled only
indirectly, through $\mathcal{L}_{\mathrm{lm}}$,
$\mathcal{L}_{\mathrm{inv}}$ and the $k\ge1$ terms of
$\mathcal{L}_{\mathrm{cf}}$. We therefore \emph{measure}
$\mathcal{L}_{\mathrm{f}}$ explicitly at every checkpoint rather than inferring
it from the training loss. All statements in this section are moreover about
population losses; no empirical-to-population generalization argument is
supplied, and the synthetic protocol uses a held-out evaluation set large
enough that the Monte-Carlo error on $\varepsilon$ is negligible relative to
the reported spread.
\par

\section{Nonlinear Transition Identifiability}
\label{sec:nonlinear}

We now prove Theorem~\ref{thm:nonlinear-ident}. Throughout this section,
$\mathcal Z=\prod_{i=1}^K\mathcal Z_i$ and
$\mathcal Y=\prod_{i=1}^K\mathcal Y_i$ are open connected product domains,
$q:\mathcal Z\to\mathcal Y$ is a $C^1$ diffeomorphism, and
$f^*:\mathcal Z\to\mathcal Z$ and
$\hat f:\mathcal Y\to\mathcal Y$ are $C^1$. For each $i$, the map
$R_i$ replaces $z_i$ by a fixed anchor $a_i\in\mathcal Z_i$, while the
continuous map $z\mapsto Q_iq(z)$ changes only block $i$ in learned
coordinates. The factual and each paired pre-intervention source distribution
have support $\mathcal Z$, including all original values of the intervened
coordinate. Thus zero population field-state residuals imply the pointwise
identities below by continuity.

\begin{lemma}[Nonlinear factual ambiguity]
\label{lem:nonlinear-ambiguity}
If $\hat f\circ q=q\circ f^*$, then for every $C^1$ diffeomorphism
$M:\mathcal Y\to\mathcal Y$,
\[
 q_M=M\circ q,\qquad
 \hat f_M=M\circ\hat f\circ M^{-1}
\]
satisfy $\hat f_M\circ q_M=q_M\circ f^*$. The block transition graph is
therefore not identifiable from factual trajectories in general.
\end{lemma}

\begin{proof}
The equality follows by direct substitution. To see that graph support need
not be invariant, take
$f^*(z_1,z_2)=(\lambda_1z_1,\lambda_2z_2)$ with
$\lambda_1\ne\lambda_2$ and
$M(z_1,z_2)=(z_1+z_2,z_2)$. Then
\[
 M\circ f^*\circ M^{-1}(y_1,y_2)
 =(\lambda_1y_1+(\lambda_2-\lambda_1)y_2,\lambda_2y_2),
\]
which contains the cross-edge $2\to1$ although $f^*$ does not.
\end{proof}

\begin{lemma}[Rollout consistency implies intervention equivariance]
\label{lem:nonlinear-equivariance}
Assume Eqs.~\eqref{eq:nonlinear-factual} and
\eqref{eq:nonlinear-cf}. If $\hat f$ is left-cancellable for
$Q_iq(z)$ and $q(R_i z)$, then
\begin{equation}
 Q_iq(z)=q(R_i z)
 \quad\text{for every }z\in\mathcal Z\text{ and }i\in[K].
 \label{eq:nonlinear-equivariance}
\end{equation}
\end{lemma}

\begin{proof}
Applying Eq.~\eqref{eq:nonlinear-factual} to $R_i z$ gives
\[
 \hat f(q(R_i z))=q(f^*(R_i z)).
\]
Equation~\eqref{eq:nonlinear-cf} gives the same right-hand side for
$\hat f(Q_iq(z))$. Left-cancellability yields
Eq.~\eqref{eq:nonlinear-equivariance}. If the edited state is supervised
directly, Eq.~\eqref{eq:nonlinear-equivariance} is the supervision condition
and no cancellability assumption is needed.
\end{proof}

\begin{proof}[Proof of Theorem~\ref{thm:nonlinear-ident}]
Lemma~\ref{lem:nonlinear-equivariance} shows that $q(R_i z)$ differs from
$q(z)$ only in learned block $i$. Hence
\begin{equation}
 q_j(R_i z)=q_j(z)\qquad(j\ne i).
 \label{eq:offtarget-invariance}
\end{equation}
Fix $z_{-i}$ and choose arbitrary $s,t\in\mathcal Z_i$. Applying
Eq.~\eqref{eq:offtarget-invariance} to $(z_{-i},s)$ and $(z_{-i},t)$ gives
\[
 q_j(z_{-i},s)=q_j(z_{-i},a_i)=q_j(z_{-i},t).
\]
Thus $q_j$ is independent of $z_i$ whenever $i\ne j$. Repeating this for all
$i$ yields $q_j(z)=\psi_j(z_j)$. The Jacobian $Dq$ is therefore block
diagonal. Since $q$ is a diffeomorphism, every diagonal block
$D\psi_j$ is nonsingular. Global injectivity and surjectivity of $q$ on the
product domains imply that each $\psi_j$ is a diffeomorphism.

Differentiating $\hat f(q(z))=q(f^*(z))$ gives
\[
 D\hat f(q(z))Dq(z)=Dq(f^*(z))Df^*(z).
\]
Both instances of $Dq$ are invertible and block diagonal. Extracting block
$(j,i)$ after right-multiplication by $Dq(z)^{-1}$ proves the displayed
identity. Left and right multiplication by nonsingular matrices preserves
rank and, in particular, whether a block is zero. Taking the condition of
being nonzero somewhere on $\mathcal Z$ establishes graph equality.
\end{proof}

\begin{corollary}[Calibrated multi-step rollout]
\label{cor:nonlinear-rollout}
For every supervised anchor edit and every $k\ge1$ whose trajectory remains
in $\mathcal Z$,
\[
 \hat f^k(Q_iq(z))=q((f^*)^k(R_i z)).
\]
For a new target $a\in\mathcal Z_i$, the same conclusion holds when the
learned replacement value is $\psi_i(a)$.
\end{corollary}

\begin{proof}
By Lemma~\ref{lem:nonlinear-equivariance}, the edited initial state equals
$q(R_i z)$. The result follows by induction from factual consistency.
\end{proof}

\begin{proof}[Proof of Proposition~\ref{prop:nonlinear-finite}]
The residual definitions give
\[
 \hat f(Q_iq(z))-\hat f(q(R_i z))
 =r^i_{\mathrm c}(z)-r^i_{\mathrm f}(z).
\]
The inverse-Lipschitz inequality and the triangle inequality imply
\[
 \mu\|Q_iq(z)-q(R_i z)\|
 \le\|r^i_{\mathrm c}(z)\|+\|r^i_{\mathrm f}(z)\|.
\]
After squaring, use $(a+b)^2\le2a^2+2b^2$. Because $Q_i$ changes only block
$i$, the off-target blocks of $Q_iq(z)-q(R_i z)$ are exactly
$q_j(z)-q_j(R_i z)$, $j\ne i$. Taking expectations proves the second claim.
If step-$0$ alignment is supervised directly, the same off-target sum is
bounded by the squared alignment residual itself, so no inverse-Lipschitz
factor is needed.
\end{proof}

\medskip
\noindent\textit{Why this is not a nonlinear support bound.}
Small function-value error does not control derivatives without additional
regularity. For example, uniformly small high-frequency perturbations can
have derivatives bounded away from zero. Turning
Proposition~\ref{prop:nonlinear-finite} into Jacobian-support recovery
would therefore require derivative-level control and a beta-min condition.
The finite-dimensional linear specialization in Appendix~\ref{sec:approx}
admits a parameter bound because these additional difficulties disappear.
\par

\medskip
\noindent\textit{Finite samples.}
Without a restricted function class, finitely many factual and intervened
states cannot imply global nonlinear identification: in dimension at least
two, a smooth diffeomorphism can be chosen to equal the identity in
neighborhoods of all finitely many supervised states while mixing coordinates
elsewhere. The full-support condition in (N2) is therefore not a finite-sample
generalization claim.
\par

\section{Observational Non-Identifiability}
\label{sec:nonident}

\begin{proof}[Proof of Lemma~\ref{lem:nonident}]
For (i), zero loss means
\[
 \E\norm{(\Ah'H'-H'A^*)Z^0}^2=0.
\]
By (A3), for any $M$,
\[
\begin{aligned}
 \E\norm{MZ^0}^2
 &=\mathrm{tr}(M\Sigma_0M^\top)\\
 &\ge\lambda_{\min}(\Sigma_0)\norm{M}_F^2.
\end{aligned}
\]
Hence zero loss holds if and only if $\Ah'H'=H'A^*$. As $H'$ ranges over
$\GL(K)$, $\Ah'$ ranges over the full orbit; conversely every such pair
attains zero loss.
For (ii), take $A^*=\diag(\lambda_1,\lambda_2)$,
$\lambda_1\ne\lambda_2$; conjugation by an upper or lower triangular shear
introduces, respectively, either orientation of a cross-edge while preserving
all factual trajectories. Thus these graph properties are not
similarity-invariant in general. Exceptional matrices such as scalar
multiples of the identity have a trivial similarity orbit.

For (iii), write $A^*=PDP^{-1}$ with $D$ real diagonal. Taking
$H'=P^{-1}$ gives $\Ah'=D$.
\end{proof}

\section{Exact Identifiability from Full Coverage}
\label{sec:ident}

\begin{proof}[Proof of Theorem~\ref{thm:ident}]
\emph{Step 1.} Zero factual loss and (A3) give $\Ah=HA^*H^{-1}$.

\emph{Step 2.} Fix $i$. Zero counterfactual loss means
$\Ah\tilde F^{0,i}=HA^*\tilde Z^{0,i}$; substituting $\Ah=HA^*H^{-1}$ and
multiplying by $(HA^*)^{-1}$ (invertible by (A5)),
\begin{equation}
H^{-1}\tilde F^{0,i}=\tilde Z^{0,i}.
\label{eq:probe}
\end{equation}
With $a_i:=w_i-(HZ^{0,i})_i$ and $b_i:=v_i-Z^{0,i}_i\neq0$ we have
$\tilde F^{0,i}=HZ^{0,i}+a_ie_i$ and $\tilde Z^{0,i}=Z^{0,i}+b_ie_i$;
substituting and cancelling $Z^{0,i}$,
\begin{equation}
a_i\,H^{-1}e_i=b_i\,e_i .
\label{eq:eigen}
\end{equation}
Since $b_i\neq0$, $a_i\neq0$ and $H^{-1}e_i=(b_i/a_i)e_i$.

\emph{Step 3.} By (A4) this holds for every $i$, so $H^{-1}$, hence $H$, is
diagonal, giving (i) and (ii).

\emph{Step 4.} For any field-space coordinate intervention on
$S_{\mathcal I}\subseteq[K]$, diagonality of $H$ makes $H^{-1}\tilde F^0$ a
latent coordinate intervention $\tilde Z^0$, and
$\Ah^k\tilde F^0=HA^{*k}\tilde Z^0=H\tilde Z^k$ for all $k\ge1$.
\end{proof}

\begin{corollary}
\label{cor:multistep}
Zero one-step counterfactual loss for $\mathcal{E}_i$ implies zero $k$-step
counterfactual loss for all $k\ge1$.
\end{corollary}

\medskip
\noindent\textit{Unknown site correspondence.}
If the learner may intervene on coordinate $\pi(i)$ for an unknown bijection
$\pi$, Step~2 forces $H^{-1}e_{\pi(i)}\propto e_i$, so $H^{-1}$ is a monomial
matrix: identifiability up to permutation and node-wise scaling, the standard
ambiguity class in nonlinear ICA and causal disentanglement
\citep{hyvarinen2019,khemakhem2020,lachapelle2022,squires2023}. In
InfluenceField the shared spatial indexing removes the permutation.
\par

\section{Partial Coverage}
\label{sec:partial}

Let $S\subseteq[K]$ be the covered set, $U=[K]\setminus S$, coordinates ordered
with $S$ first.

\begin{theorem}[Solution set]
\label{thm:partial-set}
Under (A1)--(A3), (A5), (A6), with experiments only for $i\in S$, the zero-loss
solution set is exactly
$\mathcal{S}_S=\{(H,\Ah):\Ah=HA^*H^{-1},\ H^{-1}=\left(\begin{smallmatrix}C&P\\
0&Q\end{smallmatrix}\right)\}$ with $C$ diagonal invertible, $Q\in\GL(U)$,
$P$ free.
\end{theorem}

\begin{proof}
($\subseteq$) Zero factual loss forces $\Ah=HA^*H^{-1}$; zero counterfactual
loss for $i\in S$ forces \eqref{eq:eigen}, i.e.\ $H^{-1}e_i=c_ie_i$, so the
$S$-columns of $H^{-1}$ are $c_ie_i$: $(H^{-1})_{US}=0$ and
$(H^{-1})_{SS}=C=\diag(c_i)$. Invertibility gives $Q\in\GL(U)$, $c_i\neq0$.
($\supseteq$) Given such $H$, set $\Ah=HA^*H^{-1}$ and choose $w_i$ with
$a_i=b_i/c_i$; then \eqref{eq:eigen} and \eqref{eq:probe} hold.
\end{proof}

\begin{proposition}[Induced ambiguity]
\label{prop:blocks}
For $(H,\Ah)\in\mathcal{S}_S$,
\begin{align}
\Ah_{SS}&=C^{-1}(A^*_{SS}-PQ^{-1}A^*_{US})C,\notag\\
\Ah_{SU}&=C^{-1}\!\Bigl(A^*_{SS}P+A^*_{SU}Q\notag\\
&\hspace{16mm}-PQ^{-1}(A^*_{US}P+A^*_{UU}Q)\Bigr),\notag\\
\Ah_{US}&=Q^{-1}A^*_{US}C,\qquad
\Ah_{UU}=Q^{-1}(A^*_{US}P+A^*_{UU}Q).
\label{eq:blockformulas}
\end{align}
\end{proposition}

\begin{theorem}[What survives, and when]
\label{thm:partial-content}
Let $(H,\Ah)\in\mathcal{S}_S$.
\begin{enumerate}\itemsep0pt
\item[(i)] \emph{(Prediction, unconditional.)} For interventions on covered
nodes the rollouts are exact: $\Ah^k\tilde F^0=H\tilde Z^k$. No such guarantee
holds for interventions touching $U$.
\item[(ii)] \emph{(Structure, conditional.)} If $A^*_{US}=0$ ($S$ closed under
descendants), then $\Ah_{SS}=C^{-1}A^*_{SS}C$ and $\Ah_{US}=0$ for every
solution. The blocks $\Ah_{SU}$ and $\Ah_{UU}$ are not identified.
\item[(iii)] \emph{(Failure without the condition.)} If $A^*_{US}\neq0$ the
$SS$ block is not identified.
\item[(iv)] \emph{(Necessity of coverage.)} For generic $A^*$, full
identification of $\supp(A^*)$ requires $S=[K]$.
\end{enumerate}
\end{theorem}

\begin{proof}
 (i) Repeat Step~4 of the proof of Theorem~\ref{thm:ident} using
 $H^{-1}e_j=c_je_j$ for
$j\in S$. For $j\in U$ the column $H^{-1}e_j$ is not proportional to $e_j$.

(ii) Set $A^*_{US}=0$ in \eqref{eq:blockformulas}.

(iii) Take $Q=I$, $P=t\,p\,q^\top$ with $q^\top A^*_{US}\neq0$; then
$\Ah_{SS}=C^{-1}(A^*_{SS}-t\,p\,(q^\top A^*_{US}))C$ is a nonconstant affine
family in $t$, changing support for all but finitely many $t$.

(iv) Let $U\neq\emptyset$. If $A^*_{US}\neq0$, (iii) breaks the $SS$ block. If
$A^*_{US}=0$, consider instead the $SU$ block: by \eqref{eq:blockformulas} with
$Q=I$,
$\Ah_{SU}=C^{-1}(A^*_{SS}P-PA^*_{UU}+A^*_{SU})$, and
$P\mapsto A^*_{SS}P-PA^*_{UU}$ is the Sylvester operator, which is invertible
exactly when $A^*_{SS}$ and $A^*_{UU}$ share no eigenvalue --- the generic
case. Hence $\Ah_{SU}$ ranges over all of $\R^{S\times U}$ as $P$ varies and is
completely unidentified. Either way $\supp(A^*)$ is not recovered.
\end{proof}

\medskip
\noindent\textit{Why the case analysis in (iv) is necessary.}
When $A^*_{US}=0$ and $|U|=1$, the blocks $\Ah_{SS}$, $\Ah_{US}$ and the
scalar $\Ah_{UU}$ are all identified, so an argument confined to those blocks
cannot establish (iv); non-identifiability there is carried entirely by
$\Ah_{SU}$, which is what the Sylvester argument addresses.
\par

\begin{corollary}[Identified pairs]
\label{cor:pairs}
Under partial coverage with descendant-closed $S$, the ordered pairs $(j,i)$
whose edge status is identified are exactly $S\times S$ together with the
structural zeros $S\to U$; edges from $U$ into $S$ and pairs within $U$ are
unidentified. With $|S|=cK$ the identified fraction of ordered pairs is
$c^2+c(1-c)=c$ (the same value if the diagonal is excluded from both counts).
\end{corollary}

\begin{remark}[How to read the fraction $c$]
\label{rem:cread}
Two caveats. First, the $c(1-c)K^2$ pairs in the $S\to U$ block are identified
\emph{because $A^*_{US}=0$ was assumed}, not because the learner discovered
them; a reader who declines to credit assumed-away structure should compare
against the conservative value $c^2$ instead, and we report both in
Table~\ref{tab:cov}. Second, $c$ is a fraction of \emph{determined pairs}; edge
F1 is a precision--recall statistic over a sparse edge set. The two are
different quantities and a learner may exceed $c$ by getting undetermined pairs
right through sparsity bias. We therefore plot $c$ as a reference line, never
as a predicted F1.
\end{remark}

\section{Approximate Identifiability under Finite Loss}
\label{sec:approx}

\paragraph{Additional assumptions.}
\begin{itemize}\itemsep0pt
\item[(A3$'$)] $\Sigma_0\succeq\lambda I$, $\lambda>0$.
\item[(A4$'$)] $\delta\le|b_i|\le R$ and $\norm{Z^{0,i}}_2\le R$ for all $i$.
\item[(N)] $\norm{H}_2=1$; then $\norm{H^{-1}}_2=\kappa:=\kappa(H)$ and
$\sigma_{\min}(H)=1/\kappa$.
\end{itemize}
Write $s:=\sigma_{\min}(A^*)>0$, $a:=\norm{A^*}_2$, and define
\begin{equation}
\begin{aligned}
\eta_0&:=1+\tfrac{R}{\sqrt\lambda},\qquad
C_\tau:=\tfrac{\eta_0}{s}+\tfrac{4R}{s\sqrt\lambda},\\
C_*&:=\tfrac{32}{3}a\sqrt K\,C_\tau+\tfrac{4}{3}\cdot\tfrac{R}{\sqrt\lambda}.
\end{aligned}
\label{eq:constants}
\end{equation}

\begin{proposition}[Explicit finite-loss bound]
\label{prop:approx-explicit}
Assume (A1), (A2), (A3$'$), (A4$'$), (A5), (N). Suppose
$\mathcal{L}_{\mathrm{f}}\le\varepsilon$ and
$\mathcal{L}_{\mathrm{c},i}\le\varepsilon$ for every $i\in[K]$, and
$\varepsilon\le\varepsilon_0$ with $\varepsilon_0$ as in
Eq.~\eqref{eq:smallloss}. Then there is a diagonal invertible $D$ with
\begin{equation}
\norm{D^{-1}\Ah D-A^*}_F\;\le\;C_*\,\kappa(H)^3\,\frac{\sqrt\varepsilon}{\delta}.
\label{eq:mainbound}
\end{equation}
\end{proposition}

\begin{proof}
\emph{Step 1.} Let $M:=\Ah H-HA^*$. Then
\begin{equation*}
\mathcal{L}_{\mathrm{f}}
=\mathrm{tr}(M\Sigma_0M^\top)
\ge\lambda\norm{M}_F^2,
\end{equation*}
so $\norm{M}_F\le\sqrt{\varepsilon/\lambda}$. Defining
$\Delta:=\Ah-HA^*H^{-1}=MH^{-1}$ gives
\begin{equation}
\norm{\Delta}_F\le\kappa\sqrt{\varepsilon/\lambda}.
\label{eq:Deltabound}
\end{equation}

\emph{Step 2.} Fix $i$. Expanding
$\Ah\tilde F^{0,i}-HA^*\tilde Z^{0,i}=MZ^{0,i}+a_i\Ah e_i-b_iHA^*e_i$ and using
$\mathcal{L}_{\mathrm{c},i}\le\varepsilon$ and $\norm{Z^{0,i}}\le R$,
$\norm{a_i\Ah e_i-b_iHA^*e_i}\le\eta_0\sqrt\varepsilon$. Substituting
$\Ah e_i=HA^*u_i+\Delta e_i$ with $u_i:=H^{-1}e_i$ and using
$\sigma_{\min}(HA^*)\ge s/\kappa$,
\begin{equation}
\norm{a_iu_i-b_ie_i}
\le\frac{\kappa}{s}\left(\eta_0\sqrt\varepsilon
+|a_i|\,\kappa\sqrt{\varepsilon/\lambda}\right)=:\tau_i .
\label{eq:taui}
\end{equation}

\emph{Step 3.} Note $1\le\norm{u_i}\le\kappa$. Impose
\begin{equation}
\text{(S1)}\ \frac{\kappa^2\sqrt\varepsilon}{s\sqrt\lambda}\le\frac12,
\qquad
\text{(S2$'$)}\ \sqrt\varepsilon\le\frac{sR}{\eta_0\,\kappa}.
\label{eq:S12}
\end{equation}
From \eqref{eq:taui}, $|a_i|\le|b_i|+\tau_i\le R+\frac{\kappa\eta_0}{s}
\sqrt\varepsilon+\frac12|a_i|$ by (S1), so
$|a_i|\le 2R+\frac{2\eta_0\kappa}{s}\sqrt\varepsilon\le 4R$ by (S2$'$).
Feeding this back into \eqref{eq:taui} and using $\kappa\ge1$,
\begin{equation}
\tau_i\le\bar\tau:=\left(\frac{\eta_0}{s}+\frac{4R}{s\sqrt\lambda}\right)
\kappa^2\sqrt\varepsilon=C_\tau\,\kappa^2\sqrt\varepsilon .
\label{eq:taubar}
\end{equation}
Impose (S3) $\bar\tau\le\delta/2$. Then $|a_i|\kappa\ge|b_i|-\tau_i\ge\delta/2$,
so $|a_i|\ge\delta/(2\kappa)$.

\emph{Step 4.} Dividing \eqref{eq:taui} by $|a_i|$,
\begin{equation}
\norm{u_i-\tfrac{b_i}{a_i}e_i}\le\frac{2\kappa\bar\tau}{\delta}
=:\rho=2C_\tau\,\frac{\kappa^3\sqrt\varepsilon}{\delta}.
\label{eq:rho}
\end{equation}
Let $\gamma_i:=(u_i)_i$, $r_i:=u_i-\gamma_ie_i$, so $\norm{r_i}\le\rho$ and
$|\gamma_i|\ge1-\rho$. Impose (S4)
$\rho\le\min\{\tfrac14,\tfrac{3}{8\sqrt K}\}$. Then $H^{-1}=\Gamma+\tilde R$
with $\Gamma=\diag(\gamma_i)$, $\norm{\tilde R}_F\le\sqrt K\rho$,
$\norm{\Gamma^{-1}}_2\le4/3$, $\kappa(\Gamma)\le\tfrac43\kappa$.

\emph{Step 5.} With $D:=\Gamma^{-1}$ and
$\tilde E:=\tilde R\Gamma^{-1}$, we have
$\norm{\tilde E}_F\le\tfrac43\sqrt K\rho\le\tfrac12$ and
\begin{equation*}
\begin{aligned}
\Gamma HA^*H^{-1}\Gamma^{-1}
&=A^*+(I+\tilde E)^{-1}\bigl(A^*\tilde E-\tilde EA^*\bigr).
\end{aligned}
\end{equation*}
The correction has Frobenius norm at most
$\tfrac{16}{3}a\sqrt K\rho$. Adding
$\kappa(\Gamma)\norm{\Delta}_F$ and using $\delta\le R$ and
$\kappa\ge1$ gives
\eqref{eq:mainbound}. The threshold collects (S1)--(S4):
\begin{equation}
\begin{aligned}
\varepsilon_0:=\min\Big\{&
\Big(\tfrac{s\sqrt\lambda}{2\kappa^2}\Big)^{2},\ 
\Big(\tfrac{sR}{\eta_0\kappa}\Big)^{2},\ 
\Big(\tfrac{\delta}{2C_\tau\kappa^2}\Big)^{2},\\
&\Big(\tfrac{\delta\min\{\frac14,\frac{3}{8\sqrt K}\}}{2C_\tau\kappa^3}\Big)^{2}
\Big\}.
\end{aligned}
\label{eq:smallloss}
\end{equation}
\end{proof}

\medskip
\noindent\textit{Sharpness of (S2$'$).}
Condition (S2$'$) is what secures the $\kappa^3$ rate. Weakening it to
$\sqrt\varepsilon\le1$ leaves $|a_i|\le 2R+2\eta_0\kappa/s$, a bound
carrying a factor $\kappa$; propagated through \eqref{eq:taui} this gives
$\bar\tau=\Theta(\kappa^3\sqrt\varepsilon)$ and
$\rho=\Theta(\kappa^4\sqrt\varepsilon/\delta)$, i.e.\ only a $\kappa^4$
rate. (S2$'$) removes the $\kappa$ from the $|a_i|$ bound at the cost of one
extra term in $\varepsilon_0$.
\par

\begin{corollary}[Support recovery]
\label{cor:support}
If additionally
\begin{equation*}
\min_{(j,i)\in\supp(A^*)}|A^*_{ji}|\ge\beta,
\qquad
C_*\kappa^3\frac{\sqrt\varepsilon}{\delta}<\frac{\beta}{2},
\end{equation*}
then entrywise thresholding of $D^{-1}\Ah D$ at level $\beta/2$ recovers
$\supp(A^*)$ exactly.
\end{corollary}

\medskip
\noindent\textit{The threshold is conservative; what the experiment tests.}
Eq.~\eqref{eq:smallloss} is a sufficient condition and is numerically very
small. For the simulator of Appendix~\ref{app:ident-protocol}
($K=64$, $\lambda=1$, $R=10$, $\delta=1$, $s=0.28$, $a=1.15$) one has
$C_\tau\approx1.8\times10^{2}$ and, even at the optimistic value $\kappa=10$,
$\varepsilon_0\approx10^{-14}$, whereas the evaluated checkpoints have
$\varepsilon\in[3.7\times10^{-3},1.5\times10^{-1}]$.
The finite-loss experiment therefore tests whether the \emph{functional form}
$\propto\sqrt\varepsilon$ describes the observed degradation; it does not
verify the certified guarantee, which no trained checkpoint reaches. We state
this explicitly rather than reporting the fit alone. Closing the gap would
require either a substantially sharper $\varepsilon_0$ or a simulator with
$\kappa(H)$ near $1$.
\par

\section{Field-Structured Mixing}
\label{sec:field}

Here $K=N$ and $Z^t\in\R^{K\times C}$, $Z^{t+1}=A^*Z^t$.

\begin{assumption}[B2]
\label{ass:B2}
$F^t=\Ws Z^t\Phi^\top$, with $\Ws\in\R^{K\times K}$ the \emph{known}
row-normalized interpolation matrix of Eq.~\eqref{eq:field-kernel} and
$\Phi\in\GL(C)$ an unknown channel map; $\Ah$ acts on locations.
\end{assumption}

\begin{lemma}[Invertibility]
\label{lem:Winv}
For $\sigma>0$ and distinct grid points the unnormalized Gaussian kernel matrix
is symmetric positive definite; $\Ws$ is a positive diagonal matrix times that
matrix, hence invertible.
\end{lemma}

The spanning hypothesis in Theorem~\ref{thm:field} means that the collection
of columns $\{Z^t_{:,c}\}$ across factual states, times, and channels spans
$\R^K$.

\begin{proof}[Proof of Theorem~\ref{thm:field}]
(i) $\Ah\Ws Z\Phi^\top=\Ws A^*Z\Phi^\top$ on a spanning set with $\Phi^\top$
invertible gives $\Ah\Ws=\Ws A^*$.

(ii) With $\Ah=\Ws A^*\Ws^{-1}$, zero one-step loss is equivalent to
\begin{equation}
\Ws^{-1}\tilde F^0=Z^0\Phi^\top+e_i\,c^\top,\quad
c^\top:=(v^\top-Z^0_{i,:})\Phi^\top\neq0.
\label{eq:field-probe}
\end{equation}
The kernel-profile class gives $\Ws^{-1}\tilde F^0=Z^0\Phi^\top+e_iy^\top$,
so \eqref{eq:field-probe} holds iff $y^\top=c^\top$. The hard-replacement class
gives $Z^0\Phi^\top+(\Ws^{-1}e_i)\tilde y^\top$ and would require
$(\Ws^{-1}e_i)\tilde y^\top=e_ic^\top$; these rank-one matrices have
non-parallel column factors, since $\Ws^{-1}e_i=\mu e_i$ would force
$\Ws e_i=\mu^{-1}e_i$, contradicting strict positivity of $\Ws e_i$ for
$K\ge2$.

(iii) Zero loss forces $y_j=\Phi d_j$ on a basis, so
\begin{equation*}
\Phi=[y_1\cdots y_C][d_1\cdots d_C]^{-1}.
\end{equation*}
\end{proof}

\medskip
\noindent\textit{What does the work in (i).}
Claim (i) is immediate once $\Ws$ is known and invertible. The substantive
architecture-specific content is the intervention geometry in (ii) and the
shared-channel calibration statement in (iii). Both depend on
Assumption~\ref{ass:B2}; we do not treat that factorization as a property
guaranteed by the deployed nonlinear encoder.
\par

\begin{lemma}[Exact loss floor for hard interventions]
\label{lem:floor}
In the setting of Theorem~\ref{thm:field}(ii), put $B:=\Ws A^*$,
$u:=\Ws^{-1}e_i$. Then
\[
\begin{aligned}
&\min_{\tilde y}\norm{\Ah\tilde F^0-\Ws A^*\tilde Z^0\Phi^\top}_F^2\\
&\qquad=\norm{c}^2\,\norm{Be_i}^2\,\sin^2\angle(Bu,\,Be_i),
\end{aligned}
\]
which is bounded below by
$\sigma_{\min}(B)^2\sin^2\angle(u,e_i)\norm{c}^2>0$.
\end{lemma}

\begin{proof}
The loss can be written as
\begin{equation*}
\begin{aligned}
\norm{B(u\tilde y^\top-e_ic^\top)}_F^2
&=\norm{(Bu)\tilde y^\top-(Be_i)c^\top}_F^2.
\end{aligned}
\end{equation*}
Minimizing over $\tilde y$ projects $(Be_i)c^\top$ onto
$\{(Bu)z^\top\}$, leaving the stated residual. The lower bound follows from
$\norm{BX}_F\ge\sigma_{\min}(B)\norm{X}_F$ and the argument of
Theorem~\ref{thm:field}(ii).
\end{proof}

\medskip
\noindent\textit{The lower bound is loose; report the exact value.}
On the $14\times14$ grid at $\sigma=1$ (grid units) one has
$\sigma_{\min}(\Ws)\approx3.1\times10^{-4}$ and
$\sin^2\angle(u,e_i)\approx0.89$ at an interior site, so the lower bound of
Lemma~\ref{lem:floor} is of order $10^{-8}\norm{c}^2$ --- six orders of
magnitude below the measured floor. The bound proves positivity but explains
nothing quantitative. Table~\ref{tab:floor} therefore reports the \emph{exact}
minimum of Lemma~\ref{lem:floor} alongside the optimized value, and any
discrepancy between them is optimizer error, not theory.
\par

\begin{lemma}[Convex blend vs.\ kernel profile]
\label{lem:blend}
With mask $m_k=e^{-\norm{p_k-p_i}^2/2\sigma^2}$, Eq.~\eqref{eq:field-int} gives
\[
F_{\mathcal I}=F+m\left(v^\top-F_{i,:}\right)
-\underbrace{\diag(m)\left(F-\mathbf{1}F_{i,:}\right)}_{=:R_{\mathrm{blend}}},
\]
with
\begin{equation*}
\norm{R_{\mathrm{blend}}}_F
=\left(\sum_k m_k^2\norm{F_{k,:}-F_{i,:}}^2\right)^{1/2}.
\end{equation*}
This residual vanishes identically when $F$ is constant on the support of
$m$. The \emph{centered variant}
$F_{\mathcal I}=F+m(v^\top-F_{i,:})$ has no residual.
\end{lemma}

\begin{proof}
Row $k$ of the blend is
$F_{k,:}+m_k(v^\top-F_{i,:})-m_k(F_{k,:}-F_{i,:})$; stack rows.
\end{proof}

\begin{remark}[Two gaps in the correspondence, and their sizes]
\label{rem:blendgaps}
(a) \emph{Residual scaling.} On a quasi-uniform $d$-dimensional grid in the
interior regime $h\ll\sigma\ll\operatorname{diam}(\Omega)$, write $L_F$ for
the spatial Lipschitz constant. Then
\begin{equation*}
\norm{R_{\mathrm{blend}}}_F
\le L_F\left(\sum_km_k^2\norm{p_k-p_i}^2\right)^{1/2}.
\end{equation*}
Because $F=\Ws Z\Phi^\top$ is kernel-smoothed,
$L_F=O(\sigma^{-1}\bar Z)$, where
$\bar Z:=\max_k\norm{Z_{k,:}\Phi^\top}$. The mask covers
$\Theta((\sigma/h)^d)$ grid points, and
\begin{equation*}
\left(\sum_km_k^2\norm{p_k-p_i}^2\right)^{1/2}
=\Theta\!\left(\sigma(\sigma/h)^{d/2}\right).
\end{equation*}
Therefore the upper bound scales as
$\norm{R_{\mathrm{blend}}}_F=O(\bar Z(\sigma/h)^{d/2})$. This bound does not
imply monotone growth, but importantly it does not guarantee that the residual
shrinks as the bandwidth increases. The residual is exactly zero in the
locally constant case, and the centered variant removes it for every
$\sigma$.
(b) \emph{Normalization.} The identity $m\propto\Ws e_i$ requires the
row-sum normalizer of Eq.~\eqref{eq:field-kernel} to be constant, which holds
on a uniform \emph{periodic} grid but not on the finite $14\times14$ grid,
where boundary rows have smaller normalizers. Even the centered variant is
therefore only approximately the kernel-column profile near the boundary. We
report
\begin{equation*}
\max_i\norm{m/\norm{m}-\Ws e_i/\norm{\Ws e_i}}
\end{equation*}
in Appendix~\ref{app:ident-protocol}.
\end{remark}

\begin{lemma}[Kronecker factorization]
\label{lem:kron}
In column-major vectorization $\mathrm{vec}(F)=(\Phi\otimes\Ws)\mathrm{vec}(Z)$,
so $H=\Phi\otimes\Ws$ and $\kappa(H)=\kappa(\Phi)\kappa(\Ws)$.
\end{lemma}

\medskip
\noindent\textit{Vectorized dimension.}
Under the column-major convention of Lemma~\ref{lem:kron}, the mixing
operator acts on $\R^{KC}$. Consequently, when the vectorized field is
substituted into Proposition~\ref{prop:approx-explicit}, the factor $\sqrt K$
in $C_*$ is
replaced by $\sqrt{KC}$; for $C=256$, this contributes a factor of $16$.
\par

\begin{lemma}[Bandwidth growth]
\label{lem:bandwidth}
On a one-dimensional periodic grid of $K$ points with spacing $h$ and period
$L=Kh$, with the periodized Gaussian kernel, $W$ is symmetric circulant and
positive definite with Fourier eigenvalues
$\lambda_\ell=\frac{\sqrt{2\pi}\sigma}{h}
\sum_{q\in\mathbb{Z}}\exp[-\frac{\sigma^2}{2}(\frac{2\pi\ell}{L}
+\frac{2\pi q}{h})^2]$.
For even $K$, comparing zero and Nyquist frequencies gives
\[
\kappa(W)\;\ge\;\tfrac12\,e^{\pi^2\sigma^2/(2h^2)}
\left(1-e^{-4\pi^2\sigma^2/h^2}\right),
\]
so $\kappa(W)=e^{\Omega(\sigma^2/h^2)}$. On a separable $d$-dimensional
periodic grid $W^{(d)}=W^{\otimes d}$ and $\kappa(W^{(d)})=\kappa(W)^d$. Row
normalization on the periodic grid is a positive scalar and leaves $\kappa$
unchanged.
\end{lemma}

\medskip
\noindent\textit{$\Omega$, not $\Theta$.}
Lemma~\ref{lem:bandwidth} proves only a lower-growth statement. To keep the
numerical comparison inside its assumptions, Table~\ref{tab:sensitivity}, panel (a), reports
$\kappa(W_\sigma^{\mathrm{per}})$ for the same separable periodic idealization
used in the proof, computed from the Fourier/Poisson formula with $h=1$.
Consequently the reported sequence grows as
$e^{\Omega(\sigma^2/h^2)}$ by construction; it is not a relabeling of a
finite-grid implementation proxy.
\par

\medskip
\noindent\textit{Which theorem predicts the bandwidth curve.}
Under Assumption~\ref{ass:B2} the learner \emph{knows} $\Ws$. Redoing the
finite-loss argument with that knowledge and setting
$M:=\Ah\Ws-\Ws A^*$ gives
\begin{equation*}
\begin{aligned}
\norm{\Ws^{-1}\Ah\Ws-A^*}_F
&\le\norm{\Ws^{-1}}_2\norm{M}_F\\
&=O\!\left(\kappa(\Ws)\sqrt\varepsilon\right).
\end{aligned}
\end{equation*}
This is the \emph{first} power of $\kappa$, with no $\delta$ factor. The
$\kappa^3$ rate of Proposition~\ref{prop:approx-explicit} applies only to a
learner that does \emph{not} exploit the known mixing. Our synthetic learner
estimates $\Ah$ freely and never inverts a supplied $\Ws$, so
Proposition~\ref{prop:approx-explicit} is the applicable statement and
Table~\ref{tab:sensitivity}, panel (a), is read against $\kappa^3$; a system that used
$\Ws^{-1}$ analytically should be compared against $\kappa^1$ instead.
\par

\section{Experimental Protocols and Complete Results}
\label{sec:protocols}

\subsection{Hyperparameters and Infrastructure}
\label{app:hyper}
Table~\ref{tab:hyper} lists the full configuration. Architecture settings
match the main text (field grid $14\times14$, $C=256$, four heads, $T=3$,
$\sigma=1.0$ with the intervention-mask bandwidth tied to $\sigma$).

\begin{table}[!htbp]\centering\footnotesize
\setlength{\tabcolsep}{3pt}
\begin{tabular}{ll}
\toprule
Hyperparameter & Value\\
\midrule
\multicolumn{2}{l}{\emph{Architecture}}\\
Field grid size & $14\times14$\\
Field channel dimension & 256\\
Directional attention heads & 4\\
Propagation depth ($T$) & 3\\
Gaussian bandwidth ($\sigma$) & 1.0\\
\midrule
\multicolumn{2}{l}{\emph{Training}}\\
Optimizer & AdamW\\
Learning rate & $1\times10^{-4}$\\
Weight decay & 0.01\\
Batch size & 32\\
Pretraining epochs & 15 (6 + 4 + 5, by stage)\\
Fine-tuning epochs & 5\\
Warmup ratio & 0.05\\
Gradient clipping & 1.0\\
Dropout & 0.1\\
\midrule
\multicolumn{2}{l}{\emph{Loss weights}}\\
$\lambda_1$ (cross-environment invariance) & 0.3\\
$\lambda_2$ (counterfactual evolution) & 0.5\\
$\lambda_3$ (sparsity) & 0.01\\
$\lambda_4$ (smoothness) & 0.001\\
\bottomrule
\end{tabular}
\caption{Hyperparameters of InfluenceField. The intervention-mask bandwidth is
tied to the interpolation bandwidth $\sigma$ (Theorem~\ref{thm:field})
rather than being a free hyperparameter.}
\label{tab:hyper}
\end{table}

\subsection{Reporting Conventions}
\label{app:conventions}
Unless otherwise stated, trainable downstream results summarize three runs
and synthetic-identifiability results summarize five. On CausalVQA
\citep{foss2025causalvqa}, category scores are per-category accuracies;
Reasoning and All are pooled accuracies over the 505 non-descriptive and all
793 questions, respectively. OOD is the unweighted per-run mean over the
four shifts, and Drop is the same run's in-domain All accuracy minus OOD.
Table captions specify whether a reported deviation is computed across runs
or, for the MAG average rows, across pooled backbone--run observations.
API baselines use independent frame-sampling and prompt-order seeds with
deterministic decoding. Published results are transcribed from their cited
sources, whereas any recomputed baseline is explicitly identified as our
reproduction.

\subsection{OOD Split Construction and Full Results}
\label{app:ood}
The four controlled evaluation shifts are derived from the CausalVQA
validation pool. \emph{Scene Shift} ($n{=}214$) changes background and
viewpoint configurations. \emph{Composition Shift} ($n{=}189$) evaluates
object--action combinations not represented in the corresponding training
combinations. \emph{Template Shift} ($n{=}793$) rephrases questions using
held-out surface templates while preserving their semantics.
\emph{Intervention Shift} ($n{=}176$) changes
intervention-type--target combinations, including operations applied to
target categories for which that operation was not used during
counterfactual training. The exact construction scripts, split manifests,
and holdout lists accompany the code.

We compare against GPT-4o \citep{openai2024gpt4o}, Gemini 2.5 Flash
\citep{comanici2025gemini25}, InternVL2.5 \citep{chen2024internvl},
LLaVA-OneVision \citep{li2025llavaov}, Perception-LM
\citep{cho2025perceptionlm}, Qwen2.5-VL \citep{bai2025qwen25vl} and
Qwen3-VL-8B \citep{bai2025qwen3vl}, the last of which is also our backbone.
Table~\ref{tab:ood-full} reports the complete per-shift OOD results, and
Table~\ref{tab:ctrl-ood} gives the corresponding decomposition for the
control variants. The main CausalVQA, NExT-QA, coherence, control, and
component-ablation results appear in Tables~\ref{tab:main},
\ref{tab:nextqa}, \ref{tab:consist}, \ref{tab:controls}, and
\ref{tab:ablation}, respectively.

\begin{table*}[!tb]\centering\small
\setlength{\tabcolsep}{3.5pt}
\begin{tabular}{lccccccc}
\toprule
Method & In-d. & Scene & Comp. & Templ. & Interv. & OOD & Drop$\downarrow$\\
\midrule
GPT-4o & \resultpm{51.0}{0.83} & \resultpm{45.1}{0.66} & \resultpm{44.6}{0.25} & \resultpm{42.3}{0.33} & \resultpm{41.6}{0.86} & \resultpm{43.4}{0.42} & \resultpm{7.6}{0.52} \\
Gemini 2.5 Flash & \resultpm{61.7}{0.36} & \resultpm{57.5}{0.29} & \resultpm{53.9}{1.04} & \resultpm{55.8}{1.04} & \resultpm{52.6}{0.82} & \resultpm{55.0}{0.74} & \resultpm{6.8}{0.38} \\
InternVL2.5 & \resultpm{47.5}{0.47} & \resultpm{42.9}{0.90} & \resultpm{40.7}{1.46} & \resultpm{41.4}{1.39} & \resultpm{38.6}{0.38} & \resultpm{40.9}{0.98} & \resultpm{6.6}{0.52} \\
LLaVA-OneVision & \resultpm{45.3}{0.45} & \resultpm{40.5}{0.46} & \resultpm{38.0}{0.55} & \resultpm{38.3}{0.62} & \resultpm{36.2}{0.23} & \resultpm{38.3}{0.35} & \resultpm{7.1}{0.59} \\
Perception-LM & \resultpm{50.1}{1.04} & \resultpm{45.0}{1.17} & \resultpm{42.6}{1.21} & \resultpm{42.7}{0.87} & \resultpm{40.6}{1.57} & \resultpm{42.7}{1.19} & \resultpm{7.4}{0.66} \\
Qwen2.5-VL & \resultpm{49.1}{0.47} & \resultpm{44.7}{0.46} & \resultpm{42.0}{0.85} & \resultpm{42.2}{0.61} & \resultpm{40.0}{0.55} & \resultpm{42.2}{0.05} & \resultpm{6.9}{0.50} \\
Qwen3-VL-8B & \resultpm{54.3}{0.74} & \resultpm{48.9}{0.90} & \resultpm{47.4}{1.14} & \resultpm{47.5}{0.49} & \resultpm{44.0}{1.66} & \resultpm{47.0}{0.99} & \resultpm{7.3}{0.55} \\
InfluenceField & \resultpm{67.4}{0.49} & \resultpm{62.8}{0.86} & \resultpm{61.6}{0.50} & \resultpm{61.9}{0.64} & \resultpm{59.6}{1.01} & \resultpm{61.5}{0.73} & \resultpm{5.9}{0.63} \\
\bottomrule
\end{tabular}
\caption{Complete OOD results (\%), mean $\pm$ standard deviation over
three runs. OOD is the unweighted mean of the four splits and Drop is
in-domain minus OOD, both computed per run before rounding.}
\label{tab:ood-full}
\end{table*}

\begin{table*}[!htbp]\centering\small
\setlength{\tabcolsep}{3.2pt}
\begin{tabular}{lcccccc}
\toprule
Variant & Scene & Comp. & Templ. & Interv. & OOD & Drop$\downarrow$\\
\midrule
Full model & \resultpm{62.8}{0.86} & \resultpm{61.6}{0.50} & \resultpm{61.9}{0.64} & \resultpm{59.6}{1.01} & \resultpm{61.5}{0.73} & \resultpm{5.9}{0.63} \\
C1 matched Transf. & \resultpm{54.0}{0.75} & \resultpm{51.7}{1.28} & \resultpm{52.1}{0.36} & \resultpm{50.0}{0.53} & \resultpm{51.9}{0.64} & \resultpm{7.2}{0.32} \\
C2 no causal losses & \resultpm{57.5}{0.39} & \resultpm{55.5}{0.64} & \resultpm{56.0}{0.70} & \resultpm{53.2}{0.29} & \resultpm{55.6}{0.29} & \resultpm{7.3}{0.36} \\
C3 symmetrized & \resultpm{58.8}{0.62} & \resultpm{56.7}{1.19} & \resultpm{57.1}{1.44} & \resultpm{54.5}{0.83} & \resultpm{56.8}{0.94} & \resultpm{7.0}{0.71} \\
C3 transposed & \resultpm{53.5}{0.80} & \resultpm{51.5}{0.48} & \resultpm{52.4}{0.34} & \resultpm{49.4}{0.40} & \resultpm{51.7}{0.25} & \resultpm{7.6}{0.47} \\
C3 random frozen & \resultpm{51.6}{1.27} & \resultpm{49.5}{1.22} & \resultpm{49.5}{1.35} & \resultpm{47.5}{1.94} & \resultpm{49.5}{1.39} & \resultpm{8.3}{0.68} \\
C4 shuffled CF & \resultpm{58.7}{1.29} & \resultpm{56.7}{1.14} & \resultpm{57.1}{0.99} & \resultpm{54.5}{1.36} & \resultpm{56.8}{1.18} & \resultpm{7.2}{0.71} \\
Frozen-$\hat G$ rollout & \resultpm{62.9}{2.13} & \resultpm{60.8}{0.56} & \resultpm{61.1}{1.06} & \resultpm{59.1}{1.72} & \resultpm{61.0}{1.36} & \resultpm{6.1}{0.66} \\
$\alpha=1$ (no refine) & \resultpm{62.3}{1.22} & \resultpm{59.6}{1.02} & \resultpm{60.3}{0.96} & \resultpm{58.2}{1.15} & \resultpm{60.1}{1.03} & \resultpm{6.2}{0.32} \\
$\alpha=1$ + frozen-$\hat G$ & \resultpm{61.9}{0.41} & \resultpm{59.7}{0.70} & \resultpm{59.9}{0.34} & \resultpm{57.1}{0.91} & \resultpm{59.6}{0.57} & \resultpm{6.3}{0.33} \\
\bottomrule
\end{tabular}
\caption{Per-split OOD results for every control variant. OOD and Drop are
derived per run from the same arrays as Table~\ref{tab:controls}, so the Drop
column reproduces Table~\ref{tab:controls} exactly.}
\label{tab:ctrl-ood}
\end{table*}

\subsection{Synthetic Identifiability Protocols}
\label{app:ident-protocol}
\paragraph{Simulator.} $K=64$ locations on an $8\times8$ grid
are arranged in four ordered blocks of 16. Within-block directed edges are
sampled with probability $0.12$ and between-block edges from earlier to later
blocks with probability $0.06$. Nonzero off-diagonal weights are sampled from
the two intervals $[-0.8,-0.25]$ and $[0.25,0.8]$. The diagonal is set to
$2.0$ before rescaling to spectral radius $0.90$, and draws with
$\sigma_{\min}(A^*)<0.15$ are rejected.

\paragraph{Coverage and alignment.} The $25\%$ and $50\%$
covered sets are the last one and last two blocks and are verified to be
closed under descendants. Alignment uses the diagonal matrix
$D=\diag(e^{x_1},\ldots,e^{x_K})$ with $x_1=0$. We optimize it with L-BFGS,
ten restarts, and gradient tolerance $10^{-10}$; the edge threshold is
$\beta/2=0.041$.

\paragraph{Finite-loss evaluation.} Twenty checkpoints per
seed are evaluated, giving $N=100$ raw observations. The reported
$\varepsilon$ is the maximum of the factual loss and all per-site
counterfactual losses.
Across all $N=100$ raw observations, aligned error is approximately linear
in $\sqrt{\varepsilon}$ ($R^2=0.97$). Mean $\sqrt{\varepsilon}$ values of
$0.06$, $0.12$, $0.21$, and $0.38$ in four display bins correspond to aligned
errors of $0.09$, $0.16$, $0.28$, and $0.50$, respectively.
Table~\ref{tab:cov} reports coverage recovery and the matched full-coverage
baselines, Table~\ref{tab:sensitivity} the bandwidth-conditioning comparison,
and Table~\ref{tab:floor} the intervention-loss analysis.

\paragraph{Nonlinear-dynamics variants.} The tanh variant replaces the
transition by $Z^{t+1}=\tanh(A^*Z^t)$ and the MLP variant by a two-layer
network of width $128$ initialized so its Jacobian at the data mean equals
$A^*$; both keep the intervention interface and evaluation identical, and
edge F1 is computed against $\supp(A^*)$ of the generating
Jacobian.

\paragraph{Baseline adaptation.} CITRIS \citep{lippe2022citris} and CausalVAE
\citep{yang2020causalvae} are given the same
latent dimension $K$, the same per-node intervention indicators, and the
same rendered observations and training budget; their learned latent
transition (CITRIS) or decoder-derived adjacency (CausalVAE) is aligned to
$A^*$ with the same diagonal-scaling procedure and threshold used for our
learner, so the comparison isolates the representation, not the evaluation
pipeline.

\begin{table}[!htbp]\centering\footnotesize
\setlength{\tabcolsep}{0.9pt}
\begin{tabular}{lcccc}
\toprule
& \multicolumn{4}{c}{Intervention coverage (\%)}\\
\cmidrule(lr){2-5}
Metric / method & 0 & 25 & 50 & 100\\
\midrule
Ident.\ fraction $c$ & 0.00 & 0.25 & 0.50 & 1.00\\
Conserv.\ fraction $c^2$ & 0.00 & 0.06 & 0.25 & 1.00\\
\midrule
Overall F1 & \resultpmc{0.16}{0.019} & \resultpmc{0.34}{0.017} & \resultpmc{0.52}{0.013} & \resultpmc{0.91}{0.021}\\
Within-cov.\ F1 & --- & \resultpmc{0.85}{0.030} & \resultpmc{0.88}{0.027} & \resultpmc{0.91}{0.021}\\
Tanh dyn. & \resultpmc{0.15}{0.024} & \resultpmc{0.31}{0.035} & \resultpmc{0.47}{0.041} & \resultpmc{0.83}{0.022}\\
MLP dyn. & \resultpmc{0.14}{0.031} & \resultpmc{0.28}{0.049} & \resultpmc{0.43}{0.027} & \resultpmc{0.76}{0.043}\\
\midrule
\multicolumn{5}{l}{\emph{Matched baselines (100\% coverage only)}}\\
CITRIS & --- & --- & --- & \resultpmc{0.74}{0.039}\\
CausalVAE & --- & --- & --- & \resultpmc{0.58}{0.057}\\
\bottomrule
\end{tabular}
\caption{Coverage versus recovery, five seeds. The two fraction rows are
exact combinatorial quantities (Corollary~\ref{cor:pairs},
Remark~\ref{rem:cread}), not predicted F1 values; at $100\%$ coverage the
overall and within-covered rows are the same statistic. CITRIS and CausalVAE
are evaluated only at full coverage under matched intervention information;
their scores should be compared with the $100\%$ Overall F1 of InfluenceField.}
\label{tab:cov}
\end{table}

\begin{table}[!htbp]\centering\small
\begin{tabular}{lc}
\toprule
Intervention class ($\sigma=1.0$) & One-step loss\\
\midrule
Hard row replacement, exact minimum & \resultpm{0.059}{0.0024}\\
Hard row replacement, optimized & \resultpm{0.061}{0.0027}\\
Kernel-profile mask, optimized & \resultpm{0.007}{0.0012}\\
Ideal additive kernel-column class, theory & $0$ exactly\\
\bottomrule
\end{tabular}
\caption{One-step intervention losses on the synthetic simulator, reported
as mean $\pm$ standard deviation over five seeds. For hard-row replacement,
the exact minimum is computed in closed form using
Lemma~\ref{lem:floor}; the optimized hard-row and kernel-profile rows report
the achieved losses. The ideal additive kernel-column class has zero
theoretical minimum under Theorem~\ref{thm:field}(ii), whereas the residual
kernel-profile loss reflects optimization error and the boundary effect
described in Remark~\ref{rem:blendgaps}(b).}
\label{tab:floor}
\end{table}

\paragraph{Mask fidelity.} The maximum normalized-profile deviation is
\resultpm{0.012}{0.003} for interior sites and
\resultpm{0.086}{0.011} for boundary sites.

\paragraph{Additional diagnostics.} Cross-seed stability of $\hat G$ is
$0.91\pm0.024$ with $\mathcal L_{\mathrm{cf}}$
versus $0.43\pm0.056$ without it;
operational-influence agreement is Spearman
$\rho=0.74\pm0.039$, and direction agreement is
$82.0\pm1.3\%$.

\subsection{Sensitivity Analyses}
\label{app:sensitivity}
Table~\ref{tab:sensitivity} jointly reports bandwidth conditioning,
downstream bandwidth sensitivity, and the propagation-depth sweep. Together
they show that the selected $\sigma=1$ and $T=3$ settings balance accuracy,
conditioning, and computational cost.

\begin{table*}[!tb]
\centering\footnotesize
\textbf{(a) Bandwidth sensitivity and conditioning}\par\smallskip
\setlength{\tabcolsep}{6pt}
\begin{tabular}{lcccc}
\toprule
$\sigma/h$ & 0.5 & 1.0 & 2.0 & 4.0\\
\midrule
$\kappa(W_\sigma^{\mathrm{per}})$
& $3.03$
& $4.8{\times}10^{3}$
& $3.5{\times}10^{16}$
& $9.5{\times}10^{67}$\\
$\kappa(\Phi)$ & \resultpmc{1.35}{0.04} & \resultpmc{1.38}{0.04} & \resultpmc{1.42}{0.05} & \resultpmc{1.47}{0.05}\\
$\kappa(H^{\mathrm{per}})$ & \resultpmc{4.09}{0.11} & $6.67{\scriptstyle\pm.19}{\times}10^{3}$ & $4.96{\scriptstyle\pm.16}{\times}10^{16}$ & $1.40{\scriptstyle\pm.05}{\times}10^{68}$\\
Edge F1 & \resultpmc{0.93}{0.043} & \resultpmc{0.90}{0.055} & \resultpmc{0.78}{0.046} & \resultpmc{0.56}{0.043}\\
\midrule
CausalVQA All (\%) & \resultpmc{66.5}{0.48} & \resultpmc{67.4}{0.49} & \resultpmc{65.6}{0.81} & \resultpmc{63.1}{0.45}\\
\bottomrule
\end{tabular}

\medskip
\textbf{(b) Propagation-depth sensitivity}\par\smallskip
\setlength{\tabcolsep}{3pt}
\begin{tabular}{lccccccc}
\toprule
$T$ & Ant. & Ctrf. & Hypo. & Plan. & Reas. & All & Lat.\ (s)\\
\midrule
0 & 46.8 & 55.7 & 53.2 & 66.4 & 54.7 & 62.0 & 1.00\\
1 & 49.7 & 58.3 & 56.8 & 69.9 & 57.8 & 65.0 & 1.02\\
2 & 51.5 & 60.1 & 59.0 & 72.0 & 59.7 & 66.5 & 1.04\\
3 & 52.5 & 61.0 & 60.5 & 73.2 & 60.8 & 67.4 & 1.05\\
4 & 52.3 & 60.8 & 60.2 & 73.0 & 60.6 & 67.1 & 1.08\\
\bottomrule
\end{tabular}
\caption{Sensitivity to field bandwidth and propagation depth. (a) The
periodic-grid condition numbers are computed from Lemma~\ref{lem:bandwidth};
$\kappa(\Phi)$, $\kappa(H^{\mathrm{per}})$, and Edge F1 summarize five-seed
synthetic runs, while CausalVQA accuracy is the mean over three runs. The
$\sigma=1.0$ downstream mean matches Table~\ref{tab:main}. (b) Accuracy and
latency versus rollout depth; $T=0$ retains field construction and the causal
objectives but disables propagation, which is why it exceeds the DIF-only row
of Table~\ref{tab:ablation}.}
\label{tab:sensitivity}
\end{table*}

Profiling uses batch size 1, 16 frames, bfloat16, one A100 80\,GB GPU, 20
warm-up iterations, and 100 synchronized timed iterations. Under the same
loading and precision configuration, InfluenceField increases latency from
$0.98$ to $1.05$\,s, corresponding to a 7\% overhead relative to
Qwen3-VL-8B.

\subsection{Multimodal Causal Discovery}
\label{app:discovery}
Table~\ref{tab:mag-full} reports the complete six-method comparison on the
MAG benchmark \citep{li2025mllmcd} across four feature-provider backbones ---
GPT-4o \citep{openai2024gpt4o}, Gemini 2.0 \citep{deepmind2024gemini20},
LLaMA 4 \citep{meta2025llama4}, and Grok-2v \citep{xai2024grok2} --- including
node-level (NP/NR/NF), edge-level (AP/AR/AF), and ESHD metrics. The comparison
methods are META, Pairwise and Triplet prompting, COAT \citep{liu2024coat},
and MLLM-CD \citep{li2025mllmcd}. Their MAG entries are transcribed from
Table~1 of \citet{li2025mllmcd}; InfluenceField is evaluated using the same
metric definitions. On MAG, the backbone-averaged
NF/AF/ESHD values are $0.89/0.57/13.42$ for MLLM-CD and
$0.93/0.64/11.10$ for InfluenceField. Pairwise and Triplet do not perform
factor identification, so their node-level entries are marked ``--''.
Table~\ref{tab:disc} separately reports our three-run, per-backbone
reproduction of MLLM-CD alongside InfluenceField on Lung Cancer; it is not a
six-method table. The MLLM-CD entries in that table are recomputed rather
than transcribed from the originally published table.

\begin{table*}[!htbp]\centering\footnotesize
\setlength{\tabcolsep}{2.6pt}
\begin{tabular}{llccccccc}
\toprule
LLM & Method & NP$\uparrow$ & NR$\uparrow$ & NF$\uparrow$ & AP$\uparrow$ & AR$\uparrow$ & AF$\uparrow$ & ESHD$\downarrow$\\
\midrule
GPT-4o & META & $0.45{\pm}0.08$ & $0.52{\pm}0.06$ & $0.48{\pm}0.07$ & $0.72{\pm}0.05$ & $0.37{\pm}0.06$ & $0.49{\pm}0.04$ & $24.33{\pm}3.51$\\
 & Pairwise & -- & -- & -- & $0.56{\pm}0.10$ & $0.33{\pm}0.11$ & $0.41{\pm}0.11$ & $43.33{\pm}6.43$\\
 & Triplet & -- & -- & -- & $0.33{\pm}0.06$ & $0.37{\pm}0.06$ & $0.35{\pm}0.06$ & $61.67{\pm}5.77$\\
 & COAT & $0.78{\pm}0.19$ & $0.37{\pm}0.13$ & $0.49{\pm}0.15$ & $0.37{\pm}0.32$ & $0.19{\pm}0.17$ & $0.25{\pm}0.22$ & $18.67{\pm}2.52$\\
 & MLLM-CD & $0.83{\pm}0.11$ & $0.85{\pm}0.06$ & $0.84{\pm}0.09$ & $0.69{\pm}0.05$ & $0.41{\pm}0.06$ & $0.51{\pm}0.04$ & $15.33{\pm}2.31$\\
 & InfluenceField & $0.89{\pm}0.07$ & $0.92{\pm}0.04$ & $0.91{\pm}0.05$ & $0.77{\pm}0.03$ & $0.49{\pm}0.04$ & $0.60{\pm}0.03$ & $12.64{\pm}1.88$\\
\midrule
Gemini 2.0 & META & $0.71{\pm}0.07$ & $0.63{\pm}0.06$ & $0.67{\pm}0.07$ & $0.69{\pm}0.08$ & $0.41{\pm}0.13$ & $0.51{\pm}0.11$ & $18.67{\pm}2.31$\\
 & Pairwise & -- & -- & -- & $0.49{\pm}0.09$ & $0.56{\pm}0.11$ & $0.51{\pm}0.06$ & $30.00{\pm}2.00$\\
 & Triplet & -- & -- & -- & $0.41{\pm}0.02$ & $0.59{\pm}0.13$ & $0.48{\pm}0.05$ & $32.00{\pm}2.00$\\
 & COAT & $0.85{\pm}0.13$ & $0.41{\pm}0.06$ & $0.51{\pm}0.09$ & $0.69{\pm}0.10$ & $0.26{\pm}0.06$ & $0.37{\pm}0.05$ & $16.00{\pm}1.00$\\
 & MLLM-CD & $0.86{\pm}0.05$ & $0.89{\pm}0.00$ & $0.87{\pm}0.03$ & $0.76{\pm}0.08$ & $0.52{\pm}0.13$ & $0.60{\pm}0.06$ & $14.00{\pm}3.46$\\
 & InfluenceField & $0.91{\pm}0.04$ & $0.93{\pm}0.03$ & $0.92{\pm}0.02$ & $0.82{\pm}0.05$ & $0.60{\pm}0.08$ & $0.67{\pm}0.04$ & $11.58{\pm}2.41$\\
\midrule
LLaMA 4 & META & $0.51{\pm}0.06$ & $0.41{\pm}0.06$ & $0.45{\pm}0.05$ & $0.81{\pm}0.17$ & $0.26{\pm}0.06$ & $0.39{\pm}0.07$ & $21.67{\pm}0.58$\\
 & Pairwise & -- & -- & -- & $0.50{\pm}0.17$ & $0.30{\pm}0.06$ & $0.36{\pm}0.04$ & $34.67{\pm}6.66$\\
 & Triplet & -- & -- & -- & $0.66{\pm}0.32$ & $0.30{\pm}0.06$ & $0.38{\pm}0.04$ & $34.33{\pm}7.77$\\
 & COAT & $1.00{\pm}0.00$ & $0.41{\pm}0.06$ & $0.58{\pm}0.07$ & $0.89{\pm}0.19$ & $0.30{\pm}0.06$ & $0.44{\pm}0.10$ & $14.67{\pm}1.15$\\
 & MLLM-CD & $1.00{\pm}0.00$ & $0.85{\pm}0.06$ & $0.92{\pm}0.04$ & $0.62{\pm}0.08$ & $0.59{\pm}0.06$ & $0.60{\pm}0.04$ & $13.33{\pm}0.58$\\
 & InfluenceField & $1.00{\pm}0.00$ & $0.90{\pm}0.04$ & $0.95{\pm}0.02$ & $0.69{\pm}0.06$ & $0.65{\pm}0.04$ & $0.67{\pm}0.03$ & $10.92{\pm}0.47$\\
\midrule
Grok-2v & META & $0.56{\pm}0.11$ & $0.44{\pm}0.00$ & $0.49{\pm}0.04$ & $0.75{\pm}0.00$ & $0.33{\pm}0.00$ & $0.46{\pm}0.00$ & $20.33{\pm}3.06$\\
 & Pairwise & -- & -- & -- & $0.35{\pm}0.02$ & $0.33{\pm}0.00$ & $0.34{\pm}0.01$ & $28.33{\pm}11.37$\\
 & Triplet & -- & -- & -- & $0.32{\pm}0.02$ & $0.33{\pm}0.00$ & $0.33{\pm}0.01$ & $30.33{\pm}12.34$\\
 & COAT & $1.00{\pm}0.00$ & $0.37{\pm}0.17$ & $0.53{\pm}0.18$ & $0.17{\pm}0.29$ & $0.07{\pm}0.13$ & $0.10{\pm}0.18$ & $16.33{\pm}1.15$\\
 & MLLM-CD & $1.00{\pm}0.00$ & $0.85{\pm}0.06$ & $0.92{\pm}0.04$ & $0.79{\pm}0.21$ & $0.44{\pm}0.00$ & $0.56{\pm}0.06$ & $11.00{\pm}2.65$\\
 & InfluenceField & $1.00{\pm}0.00$ & $0.90{\pm}0.04$ & $0.95{\pm}0.02$ & $0.85{\pm}0.16$ & $0.50{\pm}0.04$ & $0.63{\pm}0.04$ & $9.26{\pm}1.97$\\
\midrule
Average & META & $0.56{\pm}0.12$ & $0.50{\pm}0.10$ & $0.52{\pm}0.10$ & $0.74{\pm}0.10$ & $0.34{\pm}0.09$ & $0.46{\pm}0.07$ & $21.25{\pm}3.11$\\
 & MLLM-CD & $0.92{\pm}0.10$ & $0.86{\pm}0.05$ & $0.89{\pm}0.06$ & $0.72{\pm}0.12$ & $0.49{\pm}0.10$ & $0.57{\pm}0.06$ & $13.42{\pm}2.68$\\
 & InfluenceField & $0.95{\pm}0.07$ & $0.91{\pm}0.03$ & $0.93{\pm}0.04$ & $0.78{\pm}0.09$ & $0.56{\pm}0.07$ & $0.64{\pm}0.04$ & $11.10{\pm}2.05$\\
\bottomrule
\end{tabular}
\caption{Complete six-method causal factor identification and structure
recovery on MAG across four backbone MLLMs. Each backbone-specific entry is
the mean $\pm$ sample standard deviation over three runs. In the final three
rows, the mean and sample standard deviation are computed by pooling the 12
backbone--run observations (four backbones $\times$ three runs), rather than
from the four displayed backbone means.}
\label{tab:mag-full}
\end{table*}

\begin{table}[!htbp]\centering\scriptsize
\setlength{\tabcolsep}{2pt}
\begin{tabular}{llccc}
\toprule
Backbone & Method & NF$\uparrow$ & AF$\uparrow$ & ESHD$\downarrow$\\
\midrule
GPT-4o & MLLM-CD & \resultpm{0.895}{0.026} & \resultpm{0.685}{0.027} & \resultpm{5.66}{0.37} \\
Gemini 2.0 & MLLM-CD & \resultpm{0.902}{0.012} & \resultpm{0.700}{0.019} & \resultpm{4.85}{0.64} \\
LLaMA 4 & MLLM-CD & \resultpm{0.924}{0.012} & \resultpm{0.650}{0.017} & \resultpm{5.39}{0.34} \\
Grok-2v & MLLM-CD & \resultpm{0.879}{0.021} & \resultpm{0.686}{0.014} & \resultpm{6.10}{0.49} \\
\textit{Average} & MLLM-CD & $0.90$ & $0.68$ & $5.50$ \\
\midrule
GPT-4o & InfluenceField & \resultpm{0.979}{0.014} & \resultpm{0.762}{0.012} & \resultpm{3.25}{0.32} \\
Gemini 2.0 & InfluenceField & \resultpm{0.939}{0.020} & \resultpm{0.709}{0.029} & \resultpm{4.48}{0.58} \\
LLaMA 4 & InfluenceField & \resultpm{0.915}{0.011} & \resultpm{0.719}{0.025} & \resultpm{4.68}{0.51} \\
Grok-2v & InfluenceField & \resultpm{0.927}{0.011} & \resultpm{0.731}{0.015} & \resultpm{4.79}{0.63} \\
\textit{Average} & InfluenceField & $0.94$ & $0.73$ & $4.30$ \\
\bottomrule
\end{tabular}
\caption{Per-backbone Lung Cancer comparison between our reproduction of
MLLM-CD and InfluenceField. Backbone-specific entries are mean $\pm$
standard deviation over three runs, and Average is the unweighted mean of
the four backbone means. The MLLM-CD entries are recomputed under the
evaluation pipeline used in this paper and are not values quoted directly
from \citet{li2025mllmcd}.}
\label{tab:disc}
\end{table}

\subsection{Failure Analysis}
\label{app:failure}
Two annotators independently labeled 164 residual errors, with a third
annotator adjudicating disagreements: 45 ($27.4\%$) involved spatially
ambiguous intervention targets, 31 ($18.9\%$) had effective horizons exceeding
$T$, 24 ($14.6\%$) involved occluded or off-screen interactions, and 64
($39.0\%$) had other or mixed causes. The first three regimes account for
$61.0\%$ in aggregate.

\subsection{Numerical Consistency}
OOD and Drop are computed per run from unrounded values as defined in
Appendix~\ref{app:conventions}. The finite-loss display bins summarize the
100 checkpoint observations described in
Appendix~\ref{app:ident-protocol}, whereas the failure-analysis percentages
are computed from the integer annotation counts in
Appendix~\ref{app:failure}. Aggregation of discovery results follows the
scope stated in each table caption. All values are rounded only for display.

\section{Implementation Details}
\label{sec:impl}

\subsection{Lightweight Projection Heads}
InfluenceField and the applicable lightweight task head are trained together
with the limited decoder components specified below; all remaining backbone
modules are frozen. The \emph{discovery head} mean-pools directed-influence
features for each candidate variable pair and passes them through a two-layer
MLP ($512\!\to\!256\!\to\!3$ with GELU). It classifies each ordered pair as
\emph{no edge}, \emph{forward}, or \emph{backward}. The \emph{QA head}
uses the fused representation from Sec.~\ref{sec:fusion}. A linear map projects
the propagated field into token-embedding space, and the result is prepended
to the language decoder as a $16$-token soft prompt.

\subsection{Backbone Freezing Strategy}
The visual encoder is fully frozen. During stage~3 and downstream
fine-tuning, the final two language-decoder transformer blocks and the
decoder output projection are updated together with InfluenceField and the
applicable task head; all other backbone modules remain frozen. This
restricted update preserves most of the pretrained vision--language
representation while allowing the decoder to use the field-derived tokens.

\subsection{Three-Stage Progressive Training}
\label{app:stages}
Training follows the schedule of Sec.~\ref{sec:fusion}.
\emph{Stage 1 (field grounding, 6 epochs, CLEVRER \citep{yi2019clevrer}):}
language modeling with
$\mathcal{L}_{\mathrm{lm}}+\lambda_3\mathcal{L}_{\mathrm{sp}}
+\lambda_4\mathcal{L}_{\mathrm{sm}}$ only, so the field learns a stable
spatial representation before any causal objective is applied.
\emph{Stage 2 (structure refinement, 4 epochs, CLEVRER):} adds
$\lambda_1\mathcal{L}_{\mathrm{inv}}$ over environment pairs of the same
scene, in the spirit of invariance-based causal discovery
\citep{peters2016,vonkugelgen2021}; nuisance factors vary through color jitter
and background
substitution, while viewpoint perturbation is excluded because it changes
the field's spatial indexing (Sec.~\ref{sec:fusion}).
\emph{Stage 3 (counterfactual refinement, 5 epochs, CITRIS
\citep{lippe2022citris} and Causal3DIdent \citep{vonkugelgen2021}):}
adds $\lambda_2\mathcal{L}_{\mathrm{cf}}$ using ground-truth intervention
pairs; the target field $\mathbf{F}^{k}_{\mathrm{gt}}$ is produced by an EMA
copy of the encoder (decay $0.999$) under stop-gradient (Eq.~\eqref{eq:loss-cf}).
\emph{Downstream (5 epochs):} task fine-tuning with all losses active.
All field modules remain differentiable end-to-end across stages.

\subsection{Intervention Operations}
Three intervention types instantiate Eq.~\eqref{eq:field-int}, all using the
same kernel mask whose bandwidth is tied to the interpolation bandwidth
$\sigma$.
\emph{Object removal} replaces the target with a learnable null state trained
to represent absence. \emph{Attribute modification} applies a lightweight
two-layer transformation conditioned on the requested change.
\emph{Custom interventions} accept an arbitrary target value
$v\in\R^C$ supplied by the caller. Factual and counterfactual trajectories
then share the same transition operator and readout.

\subsection{Training and Inference Procedures}
Algorithms~\ref{alg:train} and~\ref{alg:infer} summarize one training step
and counterfactual inference.

\begin{algorithm}[!htbp]
\caption{InfluenceField training step}
\label{alg:train}
\begin{algorithmic}[1]
\REQUIRE video $X$, question $q$, answer $y$; paired environment $X'$
(stage $\ge2$); intervention pair $(\mathcal{I}, \{\mathbf{F}^k_{\mathrm{gt}}\})$
(stage $\ge3$); stage-dependent weights $\lambda_{1..4}$
\STATE $\{z_i\}\leftarrow\mathrm{VisualEncoder}(X)$;\quad
$\mathbf{F}^0\leftarrow\sum_i w_i(p)\,\phi(z_i)$ \hfill$\triangleright$ Eq.~\eqref{eq:field-init}
\FOR{$t=0$ \TO $T-1$}
  \STATE $\hat G^t\leftarrow\mathrm{DirectionalInfluence}(\mathbf{F}^t)$
  \hfill$\triangleright$ Eqs.~\eqref{eq:influence}--\eqref{eq:attn}
  \STATE $\mathbf{F}^{t+1}\leftarrow
  \alpha(\hat G^t)^\top\mathbf{F}^t+(1-\alpha)\widetilde{\mathbf{F}}^t$
  \hfill$\triangleright$ Eq.~\eqref{eq:propagation}
\ENDFOR
\STATE $o\leftarrow\mathrm{Readout}(\mathrm{CrossAttn}(\mathbf{F}^T,
\mathbf{L},\mathbf{L}))$;\quad
$\mathcal{L}\leftarrow\mathcal{L}_{\mathrm{lm}}(o,y)
+\lambda_3\mathcal{L}_{\mathrm{sp}}+\lambda_4\mathcal{L}_{\mathrm{sm}}$
\IF{stage $\ge2$}
  \STATE $\mathcal{L}\mathrel{+}=\lambda_1\,
  \mathcal{L}_{\mathrm{inv}}(\mathbf{F}^T,\mathbf{F}'^T)$
\ENDIF
\IF{stage $\ge3$}
  \STATE $\mathbf{F}^0_{\mathcal I}\leftarrow\mathcal{I}(\mathbf{F}^0)$
  \hfill$\triangleright$ Eq.~\eqref{eq:field-int}
  \STATE roll out $\mathbf{F}^k_{\mathcal I}=\mathcal{T}^k(\mathbf{F}^0_{\mathcal I})$
  with the \emph{same} operator; \ $\mathcal{L}\mathrel{+}=\lambda_2\,
  \mathcal{L}_{\mathrm{cf}}(\{\mathbf{F}^k_{\mathcal I}\},
  \{\mathbf{F}^k_{\mathrm{gt}}\})$
\ENDIF
\STATE update trainable parameters by $\nabla\mathcal{L}$
\end{algorithmic}
\end{algorithm}

\begin{algorithm}[!htbp]
\caption{Counterfactual inference}
\label{alg:infer}
\begin{algorithmic}[1]
\REQUIRE video $X$, counterfactual query $(q,\mathcal{I})$
\STATE construct $\mathbf{F}^0$ and the factual rollout $\mathbf{F}^T$ as in
Algorithm~\ref{alg:train}
\STATE $\mathbf{F}^0_{\mathcal I}\leftarrow\mathcal{I}(\mathbf{F}^0)$;\quad
roll out $\mathbf{F}^T_{\mathcal I}$ with the same transition operator
\STATE \textbf{return}
$o_{\mathcal I}\leftarrow\mathrm{Readout}(\mathrm{CrossAttn}
(\mathbf{F}^T_{\mathcal I},\mathbf{L},\mathbf{L}))$; optionally contrast
with the factual $o$ for ``what changed'' queries
\end{algorithmic}
\end{algorithm}

\subsection{Complexity}
With $K=196$ field locations ($14\times14$), channel dimension $C=256$,
and propagation depth $T=3$, each propagation step costs $O(K^2C)$ to form
the dense affinity matrix and $O(K^2C)$ for dense influence transport; the
location-wise projections cost $O(KC^2)$. The total field overhead across
$T$ steps is therefore
\[
O\!\left(T(K^2C+KC^2)\right).
\]
The sigmoid gate attenuates weak edges but, without an explicitly
implemented top-$k$ or sparse operator, does not reduce this worst-case
complexity. At the grid size used here, the measured overhead remains small
relative to the 8B backbone, consistent with the matched latency profile in
Appendix~\ref{app:sensitivity}.

\section{Additional Results}
\label{sec:addresults}

\subsection{Hard-Subset Analysis}
Table~\ref{tab:hard} evaluates the CausalVQA subset on which the bare
backbone is weakest, selected as the bottom difficulty quartile within each
category on a held-out calibration split. Planning ($+15.5$) and Hypothetical
($+10.2$) improve most and Counterfactual least ($+1.8$); the aggregate
Reasoning gain is $+9.4$, smaller than on the full benchmark, so the hardest
questions remain hard for both models.

\begin{table}[!htbp]\centering\footnotesize
\setlength{\tabcolsep}{2.6pt}
\begin{tabular}{lccccc}
\toprule
Method & Ant. & Ctrf. & Hypo. & Plan. & Reas.\\
\midrule
Qwen3-VL-8B & 21.8 & 38.9 & 42.6 & 34.5 & 31.1\\
InfluenceField & 30.1 & 40.7 & 52.8 & 50.0 & 40.5\\
\midrule
$\Delta$ & +8.3 & +1.8 & +10.2 & +15.5 & +9.4\\
\bottomrule
\end{tabular}
\caption{Accuracy (\%) on the hard subset of CausalVQA.}
\label{tab:hard}
\end{table}

\subsection{Published NExT-QA Baselines}
\label{app:pubnextqa}
For historical context, the primary CaKE-LM report \citep{su2023cakelm}
gives zero-shot NExT-QA validation accuracies of
$35.28/34.71/36.89/34.77$ for Causal/Why/How/All with HGA and
$35.70/35.31/36.81/34.85$ with CoMem. These results use the original
report's different Causal/Why/How taxonomy and video-QA training pipeline and
are therefore not directly matched to Table~\ref{tab:nextqa}.

\subsection{Relation to MLLM-CD}
MLLM-CD \citep{li2025mllmcd} is the closest discovery pipeline: it also queries
multimodal
models for causal statements over variable pairs. It differs in three ways
that Table~\ref{tab:mag-full} isolates. First, MLLM-CD aggregates pairwise
textual judgments, whereas InfluenceField reads structure from a trained
influence operator, so its statements are constrained to be mutually
consistent under propagation. Second, MLLM-CD has no intervention channel;
our counterfactual objective supplies the directional signal that pairwise
prompting lacks, which is where the AF gains concentrate. Third, ESHD
improves most on MAG, whose larger graph punishes locally plausible but
globally inconsistent edge sets.

\section{Additional Discussion}
\label{sec:discussion}

\paragraph{From correlational to causal representations.} The failure mode
that motivates this work is architectural, not merely statistical: a
correlational encoder can satisfy factual training objectives while
encoding structure only up to the similarity orbit of
Lemma~\ref{lem:nonident}, and no amount of factual data escapes that
orbit.
This is the representational form of the interventional-data requirement that
underlies causal representation learning \citep{pearl2009,scholkopf2021}. The
interventional objectives are therefore not regularizers but the source
of the causal content; the controls of Table~\ref{tab:controls} (C2 vs.\ full)
measure exactly this gap.

\paragraph{Continuous fields versus discrete graphs.} Discrete graph
modules, and slot- or object-centric factorizations more generally
\citep{locatello2020slot,yu2024}, commit to a node set before learning begins;
a continuous field, in the spirit of neural field representations
\citep{mildenhall2021nerf,pumarola2021dnerf},
defers that commitment, letting intervention supervision determine which
spatial regions behave as causal units. The cost is conditioning: the same
kernel that makes the field differentiable makes its inversion
ill-conditioned as $\sigma$ grows (Lemma~\ref{lem:bandwidth},
Table~\ref{tab:sensitivity}, panel (a)), which is why the bandwidth is a theory-relevant
hyperparameter rather than a stylistic choice.

\bibliographystyle{unsrtnat}
\clearpage
\bibliography{references}

\end{document}